\documentclass[letterpaper,11pt]{article}
\usepackage{fullpage}

\usepackage{arxiv}
\usepackage[utf8]{inputenc} 
\usepackage[T1]{fontenc}    
\usepackage{url}            
\usepackage{booktabs}       
\usepackage{amsfonts}       
\usepackage{nicefrac}       
\usepackage{microtype}      
\usepackage[dvipsnames]{xcolor}        
\usepackage{graphicx}

\usepackage[linesnumbered,ruled,vlined]{algorithm2e}
\usepackage[hidelinks,colorlinks,allcolors=blue]{hyperref}  

\usepackage{subcaption}

\usepackage{amsmath,cleveref} 
\usepackage{amsthm,amsfonts,amssymb,bbm,commath,bm}
\usepackage{enumitem}
\usepackage{cleveref}

\newtheorem{theorem}{Theorem}
\newtheorem{lemma}{Lemma}

\newtheorem{definition}{Definition}

\newcommand{\alphav}{\mathbf{\alpha}}
\newcommand{\Xf}{\mathfrak{X}}
\newcommand\E{\mathbb{E}}
\newcommand\R{\mathbb{R}}
\newcommand\T{{\scriptscriptstyle{\mathsf{T}}}}
\newcommand\Normal{\operatorname{N}}
\newcommand{\Identity}{\operatorname{I}}
\newcommand\veps{\varepsilon}
\newcommand\simiid{\ensuremath{\sim_{\operatorname{iid}}}}
\newcommand\MSE{\operatorname{MSE}}

\newcommand\ip[1]{\langle #1 \rangle}
\newcommand{\calR}{\mathcal{R}}
\newcommand{\calX}{\mathcal{X}}
\newcommand{\calY}{\mathcal{Y}}
\newcommand{\calZ}{\mathcal{Z}}
\newcommand{\calC}{\mathcal{C}}
\newcommand{\calE}{\mathcal{E}}

\newcommand{\argmin}{\operatorname*{\arg\min}}
\newcommand{\rank}{\textrm{rank}}
\newcommand{\mean}{\textrm{mean}}

\newcommand{\op}{\textrm{op}}
\newcommand{\hGamma}{\widehat{\Gamma}}
\newcommand{\WW}{\widehat{W}}
\newcommand{\BB}{\widehat{B}}

\newcommand{\bmeps}{\bm{\varepsilon}}
\newcommand{\MM}{\widehat{M}}
\newcommand{\Mo}{\widehat{M}_{0}}
\newcommand{\VV}{\widehat{V}}
\newcommand{\trace}{\mathrm{Tr}}

\makeatletter
\newcommand{\pushright}[1]{\ifmeasuring@#1\else\omit\hfill$\displaystyle#1$\fi\ignorespaces}
\newcommand{\pushleft}[1]{\ifmeasuring@#1\else\omit$\displaystyle#1$\hfill\fi\ignorespaces}
\makeatother

\usepackage[ruled,vlined]{algorithm2e} 
\usepackage{algpseudocode}

\SetCommentSty{mycommfont}

\SetKwInput{KwInput}{Input}                
\SetKwInput{KwOutput}{Output}  

\SetAlFnt{\small}
\SetAlCapFnt{\small}
\SetAlCapNameFnt{\small}
\SetAlCapHSkip{0pt}
\IncMargin{-\parindent}

\usepackage{stmaryrd}

\usepackage[style=numeric,maxnames=10,backend=bibtex]{biblatex}
\newcommand{\citet}[1]{\citeauthor*{#1}~\cite{#1}}

\title{Learning Tensor  Representations for Meta-Learning}

\author{
    Samuel Deng \\
    Columbia University\\
    \texttt{sd3013@columbia.edu} \\
    \And 
    Yilin Guo \\
    Columbia University\\
    \texttt{yg2553@columbia.edu} \\
    \And 
    Daniel Hsu \\
    Columbia University\\
    \texttt{djhsu@cs.columbia.edu} \\
    \And 
    Debmalya Mandal\thanks{Authors are alphabetically ordered and this work was done when DM was at Columbia University.}\\
    MPI-SWS \\
    \texttt{dmandal@mpi-sws.org}
}
 
\begin{document}

\maketitle





\begin{abstract}
We introduce a tensor-based model of shared representation  for meta-learning from a diverse set of tasks.  Prior works on  learning linear representations for meta-learning assume that there is a common shared representation across different tasks, and do not consider the additional task-specific observable side information. In this work, we model the meta-parameter through an order-$3$ tensor, which can adapt to the observed task features of the task. We propose two methods to estimate the underlying tensor. The first method solves a tensor regression problem and works under natural assumptions on the data generating process. The second method uses the method of moments under additional distributional assumptions and has an improved sample complexity in terms of the number of tasks.
We also focus on the meta-test phase, and consider estimating task-specific parameters on a new task. Substituting the estimated tensor from the first step allows us estimating the task-specific parameters with very few samples of the new task, thereby showing the benefits of learning tensor representations for meta-learning. Finally, through simulation and several real-world datasets, we evaluate our methods and show that it improves over previous linear models of shared representations for meta-learning.
\end{abstract}

\section{Introduction}
One of the major challenges in modern machine learning is  training a model with limited amounts of data. This is particularly important in settings where data is scarce and new data is costly to acquire. In recent years, several techniques like data augmentation, transfer learning have been proposed to address problems with limited data. The focus of this paper is meta-learning, which has turned out to be an important framework to address such problems. The main idea behind meta-learning is to design learning algorithms that can leverage prior learning experience to adapt to a new problem quickly, and learn a useful algorithm with few samples. Such approaches have been quite successful in diverse applications like natural language processing \cite{LHCG19}, robotics \cite{NKLK20}, and healthcare \cite{ZTDZ+19}.

Meta-learning algorithms are often given a family of related tasks and attempt to use few samples on a new related task by utilizing the overlap between the new test task and already seen training tasks. In that sense, a meta-learning algorithm is learning to learn on new tasks, and performance improves with experience and number of tasks~\cite{TP98}. Despite immense success, we are yet to fully understand the theoretical foundations of meta-learning algorithms. The most promising theoretical direction stems from representation learning. The main idea is that the tasks share a common shared representation and a task-specific representation \cite{TJJ20B, TJJ20}: if the shared representation is learned from training tasks, then the task-specific representation for the new task can be learned with few samples.

Current  models of shared representations for meta-learning do not take into account two observations -- (1) the training tasks are often heterogeneous, and the shared representation cannot be captured by a single parameter, (2) tasks often come with additional task-specific observable side information, and they should be part of any representation-based model of meta-learning. The first situation often arises in robotics, and various reinforcement learning environments \cite{VSHL19}, while the latter is prevalent in recommender system \cite{VTMB+17}, where items (tasks) that users rate often come with observable features.

We aim to understand meta-learning of features for settings where task-specific observable features affect the outcome. In particular, we are interested in the following questions. (1) What is the appropriate generalization of meta-learning of linear representation with task-specific observable features? (2) Moreover, given samples from $T$ tasks, how can we efficiently learn such a representation and does it improve sample efficiency on a new task?

\subsection{Contributions}

\textbf{Tensor Based Model.} We propose a tensor based model of representations for meta-learning representations for a diverse set of tasks. In particular, we model the meta-parameter through a tensor of order-$3$, which can be thought of as a multi-linear function mapping a tuple of (input feature, observed task feature, unobserved task feature) to a real-valued output. As our model considers task-specific observed features, the meta-parameter can adapt to particular task and generalizes the matrix-based linear representations proposed by \cite{TJJ20}.

\textbf{Estimation.} We first determine the identifiable component of the shared representation based model, and estimate the first two factors of the underlying tensor in the meta-training phase. We propose two methods -- (1) tensor regression based method works with natural assumptions on the data generating process, and (2) method of moments based estimation works under additional distributional assumptions, but has improved sample complexity in terms of the number of tasks.

\textbf{Meta-Test Phase.} After estimating the shared parameters, we focus on the meta-test phase, where a new task is given. We show that substituting the estimated factors from the first step provably improves error in estimating the task-specific parameters on a new task. In particular, the excess test error on the new task is bounded by $O\left(\frac{r^2}{N_2}\right)$ where $r$ is the rank of the underlying tensor and $N_2$ is the number of samples from the new task. As tensor rank $r$ can be quite small compared to the dimensions, this highlights the benefits of learning task-adaptive representations in meta-learning. Finally, through a simulated dataset and several real-world datasets, we evaluate our methods and show that it improves over previous models of learning shared representations for meta-learning. 

\subsection{Related Work}

\citet{Baxter00} was the first  to prove generalization bound for multitask learning problem. However, they considered a model of multitask learning where tasks with shared representation are sampled from a generative model. \citet{PM13, MPR16} developed general uniform-convergence based framework to analyze multitask representation learning. However, they assume oracle access to a global empirical risk minimizer. On the other hand, we provide specific algorithms and also consider task-specific side information.

The work closest to ours is \cite{TJJ20}, who proposed a linear model for learning representation in meta-learning. Our model can be thought of as a general model of theirs as we do not assume a fixed low-dimensional representation across tasks, and can adapt to observable side-information of the tasks. We also note that \citet{TJJ20B} generalized the linear model \cite{TJJ20} to consider transfer learning with general class of functions, however, they assume oracle access to a global empirical risk minimizer, and the common representation (a shared function) does not adapt to observable features of the tasks. Finally, \citet{DHKL+20} also considered the problem of learning shared representations and obtained similar results. Compared to \cite{TJJ20B}, they consider general non-linear representations, but the representation again does not depend on the observable features of the task.

Our work is also related to the conditional meta-learning framework introduced by \cite{WDC20,DPC20}. Conditional meta-learning aims to learn a conditioning function that maps task-specific side information to a meta-parameter suitable for the task. \citet{DPC20} studies a biased regularization formulation where the goal is to find task-specific parameter close to a bias vector, possibly dependent on side-information. On the other hand, \cite{WDC20} takes a structured prediction framework, and only proves generalization bounds. 
Although our framework falls within the conditional meta-learning framework, we want to understand the benefits of representation learning on a new task. 

In this work, we aim to understand meta-learning through a representation learning viewpoint. However, in recent years, several works have attempted to improve our understanding of meta-learning from other viewpoints. These include optimization \cite{Bern21,GS20}, train-validation split \cite{BCZZ+20}, and convexity \cite{SZKA20}. Additionally, there are several recent works on understanding gradient based meta-learning \cite{FAL17,FRKL19, DCGP19, BKT19, KBT19}, but their setting is very different from ours.

Finally, we use tensors to model the meta-parameter in the presence of task-specific side information. Our estimation method uses tensor regression \cite{ZLZ13, TSHK11} and tensor decomposition \cite{AGHK+14}. For tensor regression, we build upon the algorithm proposed by \cite{TSHK11}, and for tensor decomposition we use a robust version introduced by \cite{AGM14}. 

\section{Preliminaries}
We will consider standard two-stage model of meta-learning, consisting of a meta-training phase and a meta-test phase. In the meta-training stage, we see $N$ samples from $T$ training tasks and learn a meta parameter. In the meta-test stage, we see $N_2$ samples from a fixed target task (say task $0$) and learn a target-specific parameter conditioned on the meta-parameter and  features of the new task $0$. We first define our response model which specifies the particular model of shared linear representation.

\paragraph{Response model.}
There are $T$ training tasks and each task is associated with a pair of observed and unobserved task feature vector.
Task $t$ is characterized by $(Y_t, Z_t)$ where $Y_t \in \R^{d_2}$ is the observed task feature vector and $Z_t \in \R^{d_3}$ is the unobserved task feature vector for the $t$-th task. 
A sample from task $t$ is specified by a tuple $(X,Y_t,Z_t)$, where $X$ is some user feature vector. Given such a data tuple $(X,Y_t,Z_t)$ the response is given as
\begin{align}
    R & = A(X,Y_t,Z_t) + \veps
    = \sum_{i,j,k} A_{i,j,k} X_i Y_{tj} Z_{tk} + \veps \label{eq:response-model}
\end{align}
Here the noise variable $\veps \sim \Normal(0,\sigma^2)$, and $A \in \R^{d_1 \times d_2 \times d_3}$ is the system tensor which we treat as a multi-linear real-valued function on $\R^{d_1} \times \R^{d_2} \times \R^{d_3}$. Note that, our model generalizes the linear model proposed by \cite{TJJ20}, and the meta-parameter (tensor $A$) is not a fixed parameter, and adapts to the observed feature / side information for task $t$. We will assume that the tensor $A$ has CP-rank $r$ i.e. there exist matrices $A^1 \in \R^{d_1 \times r}$, $A^2 \in \R^{d_2 \times r}$, and $A^3 \in \R^{d_3 \times r}$ such that
\[
A_{i,j,k} = \sum_{s=1}^r A^1_{si} A^2_{sj} A^3_{sk}
\]
Following \cite{KB09}, we will write $A = \llbracket \Identity_r; A^1, A^2, A^3\rrbracket$ to denote the rank-$r$ decomposition of the tensor $A$. Notice that here we assume that all the singular values of the tensor $A$ is one. Without making strong assumptions on the unobserved task features $Z_t$, general singular values cannot be identified.\footnote{We provide a counter-example in the appendix.}

\paragraph{Training data.}
Let $P$ be a distribution over the feature vectors $X_i$'s which we will often refer to as user feature vectors. Let $Q$ be a joint distribution over observed and unobserved task feature vectors.
Let $\{ X_i : i \in [N] \}$ and $\{ (Y_t,Z_t) : t \in [T] \}$ be independent random variables, where $X_1,\dotsc,X_N \simiid P$ and $(Y_1,Z_1),\dotsc,(Y_T,Z_T) \simiid Q$.
Conditional on these (random) feature vectors, let $R_{1}, \dotsc, R_{N}$ be independent realizations of $R$ from the response model in \Cref{eq:response-model}, where 
\begin{equation}\label{eq:training-samples}
R_{i} = A(X_i, Y_{t(i)}, Z_{t(i)}) + \veps_i.
\end{equation} 
Here $t: [N] \rightarrow [T]$ is a mapping that specifies, for each training instance $i$, corresponding task $t(i)$. Therefore, the training data is given as $\{X_i, Y_{t(i)},  R_i\}_{i \in [N]}$. 


\paragraph{Meta-Test data.} At test time we are given a fixed task (say $0$) with observed feature $Y_0$ and unobserved feature $Z_0$. We are given $N_2$ instances from this new task,  $\{X_i, Y_0, R_i\}_{i\in[N_2]}$ where $X_1,\ldots,X_{N_{2}} \simiid P$. Our goal is to design a predictor $f \colon \R^{d_1} \times \R^{d_2} \to \R$ that maps an input feature and an observed task feature to a predicted response. We will evaluate our predictor by its mean squared error on the new task. 
\begin{align*}
    \MSE(f)
    & = \E\sbr{ (f(X,Y_0) - A(X,Y_0,Z_0) -\veps)^2 } = \sigma^2 + \E\sbr{ (f(X,Y_0) - A(X,Y_0,Z_0))^2 } .
\end{align*}
In order to design the predictor on the new task, we need estimates of tensor $A$, and unobserved task feature on the new task $Z_0$. 

\paragraph{Notations.} 
For a matrix $B \in \R^{d_1\times d_2}$ we will write $\norm{B}_\op$ to denote its operator norm, which is defined as $\norm{B}_\op = \max_{x \in \R^{d_2}} \frac{\norm{Bx}_2}{\norm{x}_2}$. For matrix $B$, we will write $\norm{B}_F = \sqrt{\sum_{i,j} B_{ij}^2}$ to denote its Frobenius norm. 
For a tensor $A \in \R^{d_1 \times d_2 \times d_3}$ we write its spectral norm $\norm{A}_\op = \max_{\norm{x}_2 = \norm{y}_2 = \norm{z}_2 = 1} \abs{A(x,y,z)}$. Like a matrix, we will also write $\norm{A}_F = \sqrt{\sum_{i,j,k} A_{i,j,k}^2}$ to denote the Frobenius norm of the tensor $A$. We sometimes use the tensor by slices, for the $3$-order tensor $A$, we denote its horizontal slices as $A_{j::}$, for $j\in [d_{1}]$.

We will use two types of special matrix products in our paper. Given matrices $A \in \R^{d_1 \times d_2}$ and $B \in \R^{d_3 \times d_4}$, the Kronecker product $A \otimes B \in \R^{d_1 d_3 \times d_2 d_4}$ is
$$
A \otimes B = [a_1 \otimes b_1\  a_1 \otimes b_2 \ldots a_{d_2} \otimes b_{d_4-1}\  a_{d_2} \otimes b_{d_4}]
$$
For matrices $A \in \R^{d_1 \times d_2}$ and $B \in \R^{d_3 \times d_2}$, their Khatri-Rao product $A \odot B \in \R^{d_1 d_3 \times d_2}$ is
$$
A \odot B = [a_1 \otimes b_1 \ a_2 \otimes b_2 \ \ldots a_{d_2} \otimes b_{d_2}]
$$
In addition, we denote standard basis vector as $\mathbf{e}_{i}$ whose coordinates are all zero, except $i$-th equals $1$. 

\section{Estimation}
We estimate the parameters of our model in two steps -- (1)  estimate the shared tensor using the meta-training data, and (2)  estimate the parameters of the test task using the meta-test data and the estimate of the shared tensor. However, it turns out that even when there is a single task in the meta-training phase, the third factor $A^3$ cannot be identified for general tensor with orthogonal factors. In the appendix, we construct an example which shows that two different tensors with identical $A^1, A^2$ but different $A^3$ leads to the same observed outcomes. Therefore, we estimate the first two factors of the tensor $A$ in the meta-training phase. In the meta-test phase, we substitute estimates of $A^1$ and $A^2$ and recover the parameters for the new task. We provide two ways to estimate the factors. The first method uses tensor regression; the second method uses the method-of-moments.

\subsection{Tensor-Regression Based Estimation}
In order to see why tensor regression can help us recover the shared tensor, we first show an alternate way to write the response $R_i$, as defined in \Cref{eq:training-samples}. 
Define
\begin{equation}\label{eq:unobserved-task-feature}
  \calZ = \begin{bmatrix} Z_1 & \cdots & Z_T \end{bmatrix}^\T \in \R^{T\times d_3}
\end{equation}
to be the matrix corresponding to the unobserved features of the $T$ training tasks. Then $A \times_3 \calZ = A(I_{d_1}, I_{d_2}, \calZ) \in \R^{d_1 \times d_2 \times T}$ is the tensor corresponding to  unobserved parameters, defined as
\begin{equation*}
  (A\times_3 \calZ)_{i,j,t} = \sum_{k=1}^d A_{i,j,k} \calZ_{t,k}.
\end{equation*}
Additionally, we define a covariate tensor $\calX_i \in \R^{d_1 \times d_2 \times T}$ corresponding to the observed features as:
\begin{equation}\label{eq:covariate-tensor}
\begin{split}
    \calX_i(\cdot,\cdot,t) =
    \begin{cases}
         X_i Y_{t(i)}^\T & \text{if $t = t(i)$} \\
         0_{d_1 \times d_2} & \text{o.w.}
    \end{cases} 
\end{split}
\end{equation}
Then, according to \Cref{eq:training-samples}, we have the following linear regression model for the $i$-th response.
\begin{equation}\label{eq::linear-regression-model-AxZ}
  R_i =  \ip{ \calX_i, A\times_3 \calZ } + \veps_i.
\end{equation}
Therefore, we can use tensor regression to get an estimate of $A \times_3 \calZ$.  Since the three factors of $A$ are $A^1, A^2$, and $A^3$, it can be easily seen that the CP-decomposition of $A\times_3 \calZ$ is $\llbracket \Identity_r; A^1, A^2, \calZ A^3\rrbracket = \llbracket G^{-1}; A^1, A^2, \calZ A^3 G\rrbracket$. Here $G$ is a diagonal matrix with $i$-th entry $1/\norm{\calZ A^3_i}_2$ and normalizes the columns of $\calZ A^3$. Because of this particular form of the tensor $A \times_3 \calZ$, we can run a tensor decomposition of the estimate of $A\times_3 \calZ$ to recover $A^1, A^2$, and $\calZ A^3 G$. However, there is a catch as we have an estimate of $A \times_3 \calZ$, instead of the exact tensor. So we need the tensor decomposition method to be \emph{robust} to the estimation error. 
Algorithm~\ref{alg:reg_dec} describes the full algorithm for recovering $A$ from the training samples.

\begin{algorithm}[!ht]
\DontPrintSemicolon
\begin{minipage}{0.9\textwidth}
\KwInput{$(X_i,Y_{t(i)}, R_i)$ for $i=1,\ldots,N$}
	\begin{enumerate}
		\item  Solve the following tensor regression problem: 
    \begin{align}\label{eq:tensor-regression-1}
        &\BB = \argmin_{B \in \R^{d_1 \times d_2 \times T} } \left\{ \frac{1}{N} \sum_{i=1}^{N} \left(R_i - \left \langle \calX_i, B \right \rangle \right)^2 + \lambda \norm{B}_{S} \right \} 
    \end{align}
 \item   Run a robust tensor decomposition of $\BB$ of CP-rank $r$:
$$
\llbracket\WW; \BB^1, \BB^2, \BB^3\rrbracket \leftarrow \textrm{Robust-Tensor-Decomposition}(\BB, r)
$$
    \end{enumerate}
    \KwOutput{Return $\widehat{A \times_3 \calZ} = \llbracket \Identity_r; \BB^1, \BB^2, \BB^3 \WW\rrbracket$.}
\end{minipage}

\caption{Tensor-Regression Based Estimation\label{alg:reg_dec}}
\end{algorithm}

\paragraph{Tensor Regression Details and Guarantees. }
Throughout this section, we will make the following assumptions about the data generating distribution.
\begin{enumerate}[label=({A}{{\arabic*}})]
    \item \label{asn:X} $X_1,\ldots,X_N \simiid \Normal(0,\Sigma)$.
    \item \label{asn:Y} $Y_1,\ldots,Y_T \simiid \Normal(0, \Sigma_y)$.
    \item \label{asn:task} For each $i$, $t(i) \sim \textrm{Unif}\set{1,\ldots,T}$. 
\end{enumerate}
\Cref{eq:tensor-regression-1} is the tensor regression step to obtain an estimate of $B = A \times_3 \calZ$. We use a regularized 
least squared regression, introduced by \citet{TSHK11}. Here $\norm[0]{B}_{S}$ is the overlapped Schatten-1 norm of the tensor $B$, which is defined as the average of mode-wise nuclear norms i.e. $\norm[0]{B}_S = 1/3 \sum_{k=1}^3 \norm[0]{B_{(k)}}_{\star}$. Since matrix nuclear norm is a convex function, the tensor regression problem stated in \Cref{eq:tensor-regression-1} is also a convex problem, and can be solved efficiently. 

For a tensor of dimension $d_1 \times d_2 \times T$, we introduce the following notation, which will appear frequently in our bounds.
\begin{equation}\label{eq:D1-defn}
D_1 = \sqrt{d_1} + \sqrt{d_2} + \sqrt{T} + \sqrt{d_1 T} + \sqrt{d_2 T} + \sqrt{d_1 d_2}
\end{equation}
The next theorem states the guarantees of the tensor regression step.
\begin{theorem}\label{thm:tensor_regression_bound}
Suppose Assumptions~\ref{asn:X}-\ref{asn:task} hold, and $N \ge O\left(\frac{\lambda_{\max}(\Sigma)}{\lambda_{\min}(\Sigma_y)}rD_1^2\right)$. Then, with probability at least $1 - e^{-\Omega(D_1^2)}$, we have
$$
\norm[1]{\BB- B}_F \le O\left(\frac{\sigma T  D_1 \sqrt{r}}{\lambda_{\min}(\Sigma_y) \lambda_{\min}(\Sigma) \sqrt{N}} \right).
$$
\end{theorem}

The full proof is provided in the supplementary material. Here we provide an overview of the main steps of the proof. Our analysis builds upon the work by \citet{TSHK11}, who analyzed the performance of tensor regression with overlapped Schatten-1 norm. The main ingredient of the proof is to show that under certain assumptions \emph{restricted strong convexity} (RSC) holds. This property was introduced by \cite{NW11} in the context of several matrix estimation problems, and ensures that the loss function has sufficient curvature to ensure consistent recovery of the unknown parameter. \citet{TSHK11} proves that when the covariate tensors $\calX_i$ are normally distributed, RSC holds with a fixed constant. For our setting, the covariate tensors are defined in \Cref{eq:covariate-tensor} and are not necessarily distributed from a multivariate Gaussian distribution. However, we can generalize the original proof of \cite{NW11} to show that under Assumptions~\ref{asn:X} ($X_i$-s are normally distributed) and~\ref{asn:Y} (tasks are sampled uniformly at random), RSC still holds for our setting, but with constant $O(1/T)$. Then we show that the parameter $\lambda$ can be chosen to be a suitably large constant to get the error bounds of \Cref{thm:tensor_regression_bound}.

\paragraph{Tensor Decomposition Details and Guarantees:}
Having recovered the tensor $A \times_{3} \calZ$, we now aim to recover the factors $A^1,A^2$, and $\calZ A^3$. Since we do not have the exact tensor $A \times_{3} \calZ$, but rather an estimate of the tensor, we apply robust tensor decomposition method (step 2 of Algorithm~\ref{alg:reg_dec}) to recover the factors of  $A \times_3 \calZ$.  
For robust tensor decomposition method, we will apply the algorithm of \cite{AGM14}. It is in general impossible to recover the factors of a noisy tensor without making any assumptions. So we will make the following assumptions about the underlying tensor $A = [\Identity_r; A^1, A^2, A^3]$. We will write $d = \max\{d_1, d_2, T\}$.
\begin{enumerate}[label=({B}{{\arabic*}})]
    \item  The columns of the factors of $A$ are orthogonal i.e., $\langle A^1_i, A^1_j \rangle = \langle A^2_i, A^2_j \rangle = \langle A^3_i, A^3_j \rangle = 0$ for all $i \neq j$.
    \item The components have bounded $2\rightarrow p$ for some $p$ i.e. 
    $\exists p < 3$,
    $$ \max\set{\norm[1]{A^{1\T}}_{2 \rightarrow p}, \norm[1]{}1]{A^{2\T}}_{2 \rightarrow p}, \norm[1]{A^{3\T}}_{2 \rightarrow p}} \le 1 + o(1) . $$ 
    \item Rank is bounded i.e. $r = o(d)$.
\end{enumerate}

Additionally, recall the definition of $\calZ$, the matrix of unobserved features.
\begin{equation*}
  \calZ = \begin{bmatrix} Z_1 & \cdots & Z_T \end{bmatrix}^\T \in \R^{T\times d_3}
\end{equation*}
\begin{enumerate}[label=({Z}{{\arabic*}})]
\item $\frac{1}{d_3^{0.5+\gamma}}\Identity_{d_3} \preccurlyeq \calZ^\T \calZ \preccurlyeq \frac{1}{\sqrt{d_3}}\Identity_{d_3}$ for some $\gamma > 0$.
\item $\kappa(\calZ^\T \calZ) = \frac{\lambda_{\max}(\calZ^\T \calZ)}{\lambda_{\min}(\calZ^\T \calZ)} \le 1 + O(\sqrt{r/d})$.
\end{enumerate}

Although assumptions (Z1) and (Z2) might seem strong requirements on the matrix of unobserved features, they are usually satisfied when the unobserved task feature matrix is drawn from gaussian distribution. For example, if $Z_t \simiid \Normal(0, \nu \Identity_{d_3})$ then the assumptions hold for small enough $\nu$. 

\begin{lemma}[Informal Statement]\label{lem:tensor_decomp_guarantees}
Suppose tensor $A$ satisfies the assumptions (B1)-(B3),  the matrix of unobserved features $\calZ$ satisfies assumptions (Z1)-(Z2), and $N \ge \tilde{O}\left(\frac{\sigma^2 T^2  D_1^2 r }{\lambda_{\min}^2(\Sigma_y) \lambda^2_{\min}(\Sigma) } \right)$. Then the tensor $\hat{A} = [\Identity_r; \widehat{A^1},\widehat{A^2}, \widehat{\calZ A^3}]$ output by Algorithm~\ref{alg:reg_dec} satisfies
\begin{align*}
&\max\left\{ \norm[1]{\widehat{A^1} - A^1}_F, \norm[1]{\widehat{A^2} - A^2}_F\right\} \le \tilde{O}\left( \frac{\sigma T D_1 r}{\rho \sqrt{N}}\right), \ \ \norm[1]{\widehat{\calZ A^3} - \calZ A^3}_F \le  \tilde{O}\left( \frac{\sigma T D_1 r^{1.5}}{\rho \sqrt{ N}}\right)
\end{align*}
where $\rho = \sqrt{\lambda_{\min}(\calZ^\T \calZ) } \lambda_{\min}(\Sigma_y) \lambda_{\min}(\Sigma)$.
\end{lemma}

The proof shows that when the assumptions (B1)-(B3) and (Z1)-(Z3) are satisfied, we can apply robust tensor decomposition method to the tensor $A \times_3 \calZ$. Note that the bound for the third factor $\calZ A^3$ is worse by a factor of $\sqrt{r}$. This is because we recover an estimate of $\calZ A^3G$ for a diagonal matrix $G$ from tensor decomposition and then post-multiply this estimate by another diagonal matrix to obtain an estimate of $\calZ A^3$.

\subsubsection{Meta-Test}\label{sec:meta-test-1}
During the meta-test phase, we are given a new task (i.e. task $0$ with observed feature $Y_0 \in\R^{d_2}$, and hidden feature $Z_0 \in \R^{d_3}$), and our goal is to learn the unobserved parameter of this task with as few samples as possible. As is standard in the meta-learning literature, we get a new training sample from the new task, and our goal is to perform well on the test sample drawn from the new task. There are $N_2$ training samples from the new task, where the features $X_1,\ldots,X_{N_2}$ are drawn iid from a distribution $P$. We will assume each feature $X_i$ is mean-zero, has covariance matrix $\Sigma$ ($\E[X_i X_i^\T] = \Sigma$), and $\Sigma$-subgaussian i.e. $\E[\exp(v^\T X_i)] \le \exp\left(1/2\norm[0]{\Sigma^{1/2} v}^2_2\right)$. The observed responses on these $N_2$ points are given as
$$
R_i = A(X_i,Y_0,Z_0) + \veps_i
$$
where $\veps_i \sim_{\textrm{iid}} \Normal(0,1)$. Define $\calX_0$ to be the matrix corresponding to the features on the new task i.e.
\begin{equation*}
  \calX_0 = \begin{bmatrix} X_1 & \cdots & X_{N_2} \end{bmatrix}^T \in \R^{N_2\times d_1}
\end{equation*}

We aim to estimate $A^{3\T} Z_0$ by substituting the estimates of $A^1$ and $A^2$.
Notice that the response $R_i$ can also be expressed as
$$
R_i = (Y_0^\T A^2 \odot X_i^\T A^1) A^{3\T} Z_0 + \veps_i.
$$
Therefore, we can solve the following least square regression problem.
\begin{equation}\label{eq:meta-test-1}
\widehat{A^{3\T} Z_0} = \argmin_{\alpha_0 \in \R^r} \norm{R - (Y_0^\T \hat{A^2} \odot \calX_0 \hat{A^1}) \alpha_0}_2^2.
\end{equation}
If we write $\VV = (Y_0^\T \hat{A^2} \odot \calX_0 \hat{A^1})$, then the solution of Problem ~\ref{eq:meta-test-1} 
is given as
 $\widehat{A^{3\T} Z_0} = ( \VV^\T  \VV )^{-1} \VV^\T  R$. Now our prediction on a new test instance $X_0$ from the new task is given as $(Y_0^\T \hat{A^2} \odot X_0^\T \hat{A^1}) \widehat{A^{3\T} Z_0}$. With slight abuse of notation we will write this prediction as $\hat{A}(X_0, Y_0, \widehat{Z}_0)$. The next theorem bounds the mean squared error in the meta-test phase.


\begin{theorem}[Informal Statement]\label{thm:meta-test-1}
Suppose  $\max\{\norm[0]{\hat{A^1} - A^1}_F, \norm[0]{\hat{A^2} - A^2}_F \} \le \delta$. Additionally, $N_2 \ge \tilde{O}(r)$ and $\abs[0]{Y_0^\T \hat{A^2_i}} \ge \eta \norm[0]{Y_0}_2$ for all $i \in [r]$. Then for dimension-independent constants $B_1$, and $B_2$ we have
\begin{align*}&\E_{X_0}\left[ \left(R_0 - \hat{A}(X_0, Y_0, \widehat{Z}_0) \right)^2\right] = O\left(\sigma^2 + \frac{B_1}{\eta^2} r^2 \delta^2 + \frac{B_2}{\eta^2} \frac{r^2}{N_2} \right)
\end{align*}
with high probability.
\end{theorem}

The proof of the theorem shows that the mean squared error can be bounded as $O\del[0]{ r \norm[0]{\widehat{A^{3\T} {Z_0}} - A^{3\T} Z_0}_2^2 } + O\del[0]{ \delta^2 \norm[0]{\widehat{A^{3\T} Z_0}}_2^2 }$. Then we write down the first term as a sum of bias and variance term and establish respective bounds of $O(r^2/N_2)$ and $O(r^2 \delta^2)$. Finally, we show that the $L_2$-norm of $\widehat{A^{3\T} Z_0}$ cannot be too large and is bounded by $O(r)$. Substituting these three bounds on the upper bound on the mean squared error gives us the desired result. Note that the theorem requires a lower bound on the inner product between the new task feature $Y_0$ and the columns of $\hat{A^2}$. This can be avoided with a slightly worse dependence on $r$. First, we can eliminate all columns $i$ such that $\abs[0]{Y_0^\T \hat{A^2_i}} \ge \eta \norm[0]{Y_0}_2$. If there are $r'$ such columns, we work with a tensor of rank $r-r'$ in the meta-test phase. The reduction in rank increases mean squared error by at most $O(r^2 \eta^2)$. Now if we choose $\eta = \sigma / r$ we get a bound of $O(B_1 r^4 \delta^2 + B_2 r^4/N_2)$ on the excess error.

This theorem implies that for a new task, the number of samples needed is $N_2 = O(r^2/\epsilon)$ if we want to achieve a test error of $\epsilon$ on the new task. If we were to run a least squares regression on the new task from scratch, the required number of samples would have been $O((d_1 + d_2)/\epsilon)$. As the CP-rank of the tensor $A$ can be smaller (often a constant) than the dimension of the unobserved features, transfer of the knowledge of the tensor $A$ provides a significant reduction in the number of samples on the new task. 

\section{Method-of-Moments Based Estimation}
In this section, we provide a new algorithm that estimates the underlying tensor $A$ and also has optimal dependence on the number of tasks ($T$) under some additional distributional assumptions. In particular, we will assume $X_i \simiid \Normal(0, \Identity_{d_1})$, $Y_t \simiid \Normal(0, \Identity_{d_2})$, and $Z_t \simiid \Normal(0, \Identity_{d_3})$. Our algorithm is based on repeated applications a method-of-moments based estimator proposed in \cite{TJJ20}, and we  briefly summarize that estimator.
Suppose the $i$-th response is given as $R_i = X_i^\T B \alpha_{t(i)} + \veps_i$ and each $X_i \simiid \Normal(0, I_{d_1})$, and $B \in \R^{d_1 \times r}$ has orthonormal columns. 
 Then it is possible to recover $B$ from the top $r$ singular values of the statistic $\frac{1}{N}\sum_{i=1}^N R_i^2 X_i X_i^\T $. 

\paragraph{Recovering $A^1$.} For our setting, the $i$-th response is given as $R_i = A(X_i, Y_{t(i)}, Z_{t(i)}) + \veps_i$. If we want to recover the first factor $A^1$ then we can rewrite the $i$-th response as 
\begin{align*}
R_i = X_i^\T A_{(1)} (Z_{t(i)} \otimes Y_{t(i)}) + \veps_i = X_i^\T A^{1} \underbrace{(A^{3} \odot A^{2})^\T (Z_{t(i)} \otimes Y_{t(i)})}_{:=\alpha_{t(i)}} + \veps_i.
\end{align*}
Since each $X_i$ is drawn from a standard normal distribution, we can recover $A^1$ from the top-$r$ singular values of the statistic $\frac{1}{N}\sum_{i=1}^N R_i^2 X_i X_i^\T$.

\paragraph{Recovering $A^2$.} We can recover $A^2$ through a similar method. We can rewrite the $i$-th response as
\begin{align*}
 R_i = Y_{t(i)}^\T A_{(2)} (Z_{t(i)} \otimes X_i) + \veps_i = Y_{t(i)}^\T A^{2} \underbrace{W (A^{3} \odot A^{1})^\T (Z_{t(i)} \otimes X_i)}_{:= \alpha_{t(i)}} + \veps_i
 \end{align*}
Since each $ Y_{t}$ is drawn from a standard normal distribution, we can recover $A^2$ from the top-$r$ singular values of the statistic $\frac{1}{N}\sum_{i=1}^N R_i^2 Y_{t(i)} Y_{t(i)}^\T$.


\begin{algorithm}[!ht]
\DontPrintSemicolon
\begin{minipage}{0.9\textwidth}
\KwInput{$(X_i,Y_{t(i)}, R_i)$ for $i=1,\ldots,N$.}
\begin{enumerate}
    \item $UDU^\T \leftarrow \textrm{top}-r \textrm{ SVD of } \frac{1}{N}\sum_{i=1}^N R_i^2 X_i X_i^\T$. Set $\hat{A}^1 = U$.
    \item $UDU^\T \leftarrow \textrm{top}-r \textrm{ SVD of } \frac{1}{N}\sum_{i=1}^N R_i^2 Y_{t(i)} Y_{t(i)}^\T$. Set $\hat{A}^2 = U$.
\end{enumerate}
\KwOutput{Return $\hat{A^1}$ and $\hat{A^2}$.}
\end{minipage}

\caption{Method-of-Moments Based Estimation\label{alg:mom_est}}
\end{algorithm}

\begin{theorem}\label{thm:method_of_moments}
Suppose $X_i \simiid \Normal(0, \Identity_{d_1})$, $Y_t \simiid \Normal(0, I_{d_2})$, and $Z_t \simiid \Normal(0, I_{d_3})$. Then the factors $\hat{A^1}$ and $\hat{A^2}$ returned by Algorithm~\ref{alg:mom_est} satisfies  the following guarantees
\begin{align*}
    \sin \theta (\hat{A^1}, A^1) \le O\left( \sqrt{\frac{d_1 r}{TN}} \right), \ \textrm{and}\ \sin \theta (\hat{A^2}, A^2) \le O\left(\sqrt{\frac{d_2 r}{N}} \right)
\end{align*}
with probability at least $1 - T \exp(-\Omega(\min\{d_1,d_2\}))$.
\end{theorem}
Once Algorithm~\ref{alg:mom_est} estimates $\hat{A^1}$ and $\hat{A^2}$, we  again estimate $A^{3\T} Z_0$ in the meta-test phase. We can show a meta-test theorem similar to \Cref{thm:meta-test-1}, and the details are provided in the appendix.

\section{Experiments}
\label{sec:experiments}

We first evaluate our tensor-based representation learning through a simulation setup. For this experiment, we generated data from a low-rank tensor of order-$3$. We chose a tensor of dimension $100 \times 50 \times 50$ and of CP-rank $10$. We generated a training dataset of $N=1000$ points and estimated the factors $A^1$ and $A^2$ using both the tensor regression (Algorithm~\ref{alg:reg_dec}) and the method of moments (Algorithm~\ref{alg:mom_est}). For the meta-test phase, we selected a new test task with observed feature $Y_0$ of dimension $50$ and unobserved feature $Z_0$ of dimension $50$. As described in \Cref{sec:meta-test-1}, we estimate $\widehat{A^{3\T} Z_0}$ by substituting the estimated factors from the meta-training step. 



We plot the meta-test error for various values of $N_2$, the number of samples available from the new task. As we increase $N_2$, test error for predicting outcome on a new test instance $X_0$ decreases significantly, as shown in \Cref{fig:synthetic}. We compare our method with the matrix-based representations for meta learning developed by \cite{TJJ20}. They assume that the response from a task $t$ with unobserved feature $Z_t \in \R^r$ and $i$-th feature $X_i$ is given as
\begin{equation*}
    R_i = X_i B Z_t + \veps_i
\end{equation*}
where matrix $B \in \R^{d\times r}$. Recall that, for our setting, each training instance is given as $(X_i, Y_{t(i)}, R_i)$. Since \cite{TJJ20} assume that there is no available side-information for the tasks, the most natural comparison would be to ignore the observable task features $Y_t$ and consider each input as $(X_i, R_i)$. So we consider two natural dimensions of the matrix $B$. First, we estimate a matrix of dimension $d_1 d_2 \times d_3$ where $X_i \otimes Y_{t(i)}$ is the $i$-th feature. Second, we estimate a matrix of dimension $(d_1 + d_2) \times d_3$ where $[X_i Y_{t(i)}]$ is the $i$-th input feature. 
We compare these two different types of matrix based methods with both tensor regression and method-of-moments based method. As \Cref{fig:synthetic} shows both tensor methods perform equally well, but they are significantly better than the matrix methods.

\begin{figure}
\centering
    \begin{minipage}[!b]{0.32\linewidth}
     \includegraphics[width=\linewidth]{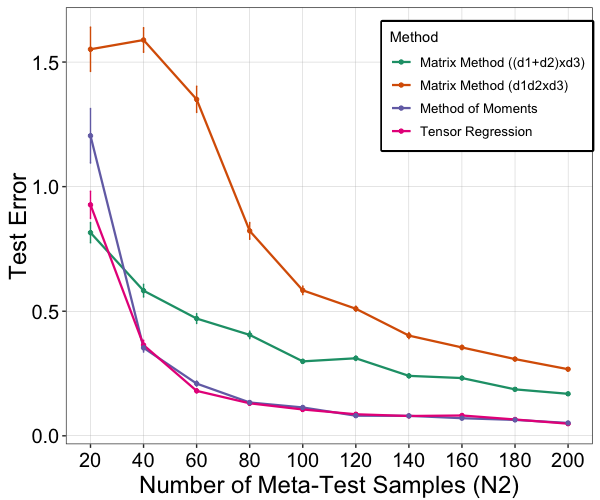}
     \subcaption{Synthetic Dataset \label{fig:synthetic}}
    \end{minipage}\hspace{0.01\linewidth}%
    \begin{minipage}[!b]{0.32\linewidth}
     \includegraphics[width=\linewidth]{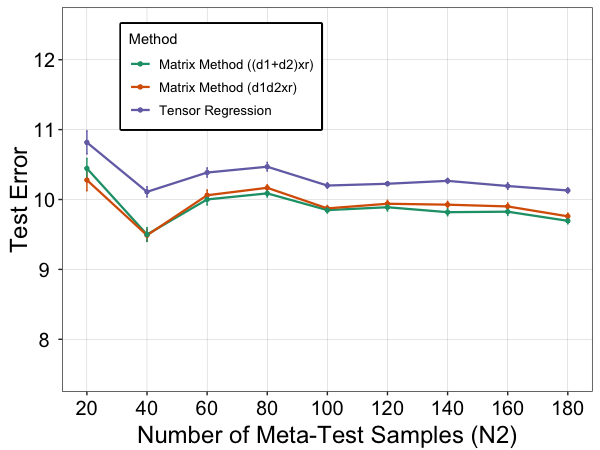}
     \subcaption{Schools Dataset \label{fig:school}}
    \end{minipage}
    \hspace{0.01\linewidth}%
    \begin{minipage}[!b]{0.32\linewidth}
     \includegraphics[width=\linewidth]{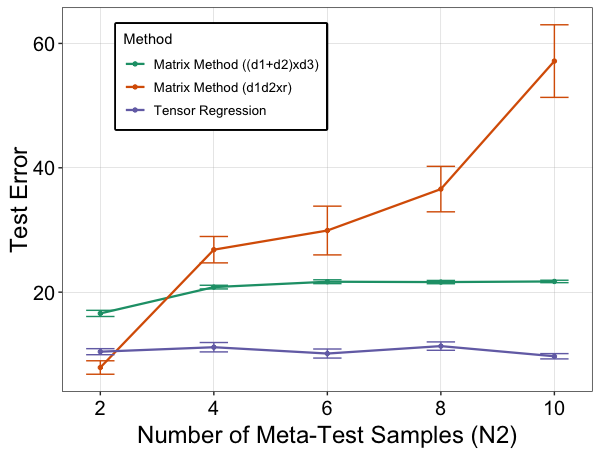}
     \subcaption{Lenk Dataset \label{fig:lenk}}
    \end{minipage}
    \caption{Test-error vs the number of samples from a new task ($N_2$) \label{fig:N2_vs_MSE}}
\end{figure}

We now consider two real-world datasets. Both the datasets were used in the context of \emph{conditional meta-learning} to show the benefits of task-specific side-information \cite{DPC20}.

\textbf{Schools Dataset} \cite{AEP08}.  This dataset consists of examination records from $T = 139$ schools (task). The number of samples per task ($n_t$) varied from $24$ to $251$. Each instance represents an individual student, and is represented by a feature of dimension $d_1 = 26$. The outcomes are their exam scores. As task specific feature of task $t$ we use $Y_t = \frac{1}{n_t} \sum_{i=1}^{n_t} \phi(x_i)$ where $\phi(x_i)$ is a vector of dimension $d_2=50$ constructed from a random Fourier feature map. This is built as follows. First sample $v$ from $\textrm{Unif}[0,2\pi]^{d_2}$. Then a matrix $U \in \R^{d_2 \times d_1}$ is sampled from $\Normal(0,\sigma^2 \Identity)$. Finally, we set 
$$
\phi(x_i) = \sqrt{\frac{2}{d_2}} \cos\left( Ux_i + v\right) \in \R^{d_2}.
$$

\textbf{Lenk Dataset} \cite{MPS16, LDGY96}. This is a computer survey data where  $T=180$ people (tasks) rated the likelihood of purchasing one of $20$ different personal computers. So there are $20$ different samples from each task. The input has dimension $d_1 = 13$ and represents different computers’ characteristics, while the output is an integer rating from $0$ to $10$. As task specific feature of a task $t$ we use $Y_t = \frac{1}{20} \sum_{i=1}^{20} \phi(z_i)$ where $\phi(z_i) = \textrm{vec}(x_i(R_i,1)^\T)$; the sum is over all $z_i$-s belonging to the task $t$. 

To construct the meta-training set, we sampled $50$ tasks uniformly at random and then sampled $n_t$ ($n_t = 20$ for Schools and $n_t =10$ for Lenk) responses from each task. Since we do not know the value of $r$, we also constructed a meta-evaluation set by selecting another set of $n_t$ samples from the selected tasks. The meta-evaluation set was used to select the best value of $r$ during meta-training phase. The meta-test set was constructed by selecting a fixed task and then gradually increasing the number of samples from that task. Figures~\ref{fig:school} and \ref{fig:lenk} respectively compare our method with two different types of matrix based representation learning for different values of $N_2$. We found that the tensor regression method performs better than the method-of-moments based estimator and only results for Algorithm~\ref{alg:reg_dec} are shown. Our method performs significantly better than the matrix based methods for Lenk. Although our method performs slightly worse on Schools, the test error increases by at most $5\%$. Overall, the performance on the synthetic dataset and two real-world datasets demonstrate the benefits of using tensor based representations for meta-learning.

\section{Conclusion and Open Questions}\label{sec:conclusion}
In this work, we develop a tensor-based model of shared representation for learning from a diverse set of tasks. The main difference with previous models on shared representations for meta-learning is that our model incorporates the observable side information of the tasks. We designed two methods to estimate the underlying tensor and compared them in terms of recovery guarantees, required assumptions on the tensor, and mean squared error on a new task. 

There are many interesting directions for future work. An interesting direction is to generalize our model and consider non-linear models of shared representations that incorporates the observable side-information of the tasks. Finally, we just leveraged the framework of order-$3$ tensor in this work, and it would be interesting to see if we can leverage higher order tensors for learning shared representations for meta-learning.


\printbibliography



\newpage

\appendix

\onecolumn

\section{Non-identifiability of General Model}
We first show that the general response model is not identifiable unless all the singular values are one.
\begin{lemma}
Consider the response model $R_i = A(X_i, Y_{t(i)}, Z_{t(i)}) + \veps_i$ specified in \eqref{eq:response-model}. Then the underlying tensor $A$ is not identifiable if the singular values are not all ones.
\end{lemma}

\begin{proof}
We show that the statement is also true for simpler matrix based linear representations for multitask learning. In that case, the responses are generated as $R = x^\top B z$ for an orthonormal matrix $B$. Now consider the model $R = x^\top B W z$ for a diagonal matrix $W$. Even if we assume that $\norm{z}_2 = 1$, given a choice of $W$ and $z$, one can choose $W' \neq W$ and $z' \neq z$ s.t. $x^\top B W z = x^\top B  W' z'$. A possible choice is $W'(1,1) = \lambda_1 W(1,1)$, $W'(2,2) = W(2,2)/\lambda_2$, $z'_1 = z_1/\lambda_1$, $z'_2 = \lambda_2 z_2$ and $z_1^2/z_2^2 = (\lambda_2^2 - 1) / (1 - 1/\lambda_1^2)$. Note that this choice guarantees that $\norm{z'}_2 = 1$.
\end{proof}

\begin{lemma}
Consider the response model $R_i = A(X_i, Y_{t(i)}, Z_{t(i)}) + \veps_i$ specified in \eqref{eq:response-model}. Then it is impossible to approximate $A^3$ either in terms of Frobenius norm or in terms of $\sin \theta$ distance. 
\end{lemma}

\begin{proof}
We construct an example where $d_1 = d_2 = d_3 = d$ and rank $r = d/2$. First consider the tensor $A = \llbracket \Identity_r; A^1, A^2, A^3 \rrbracket$ where $A^1 = A^2 = A^3 = \left[\begin{array}{c}
     \Identity_{r \times r}  \\
     0_{(d-r) \times r} 
\end{array} \right]$. Suppose the observed feature vector $Y = \left(\frac{1}{\sqrt{d}},\frac{1}{\sqrt{d}},\ldots, \frac{1}{\sqrt{d}} \right)$ and the unobserved feature vector $Z_1 = \left(\frac{1}{\sqrt{r}},\frac{1}{\sqrt{r}},\ldots, \frac{1}{\sqrt{r}}, 0, \ldots, 0\right)$. Then for any feature vector $X$ the expected response on this task is given as
\[
R = A(X, Y, Z_1) = \sum_{i=1}^r X(i) Y(i) Z_1(i) = \frac{1}{\sqrt{dr}} \sum_{i=1}^r X(i) = \frac{\sqrt{2}}{d} \sum_{i=1}^r X(i)
\]
where the last equality uses $r = d/2$. We now consider a new tensor $B = \llbracket \Identity_r; A^1, A^2, B^3\rrbracket$. The first two factors of $B$ are the same as the first two factors of $A$, but the third factor is different. Let $C^3$ be a bidiagonal matrix of dimension $(r+1)\times r$ with the leading diagonal and the diagonal entries just below the leading diagonal entries consisting of all ones. 
\[
C^3 = \begin{bmatrix}
\frac{1}{\sqrt{2}} & 0 & 0 & \dots & 0 \\
\frac{1}{\sqrt{2}} & \frac{1}{\sqrt{2}} & 0 & \dots & 0 \\
0 & \frac{1}{\sqrt{2}} & \frac{1}{\sqrt{2}} & \dots & 0 \\
\vdots & \vdots & \vdots &\ddots & \vdots \\
0 & 0 & 0 & \dots & \frac{1}{\sqrt{2}}
\end{bmatrix}
\]
Then $B^3 = \left[\begin{array}{c}
     C^3  \\
     0_{(r-1) \times r} 
\end{array} \right]$. The observed task feature $Y$ remains as it was but the new unobserved task feature is given as $Z_2 = \left(\frac{1}{\sqrt{d}},\frac{1}{\sqrt{d}},\ldots, \frac{1}{\sqrt{d}} \right)$. Then it can be checked that the new responses are given as
\[
R = B(X, Y, Z_2) = \sum_{i=1}^r X(i) \frac{1}{\sqrt{d}}\left(\frac{1}{\sqrt{2d}} + \frac{1}{\sqrt{2d}} \right) = \frac{\sqrt{2}}{d} \sum_{i=1}^r X(i)
\]
Therefore, we have two instances where the first two factors of the underlying tensors ($A^1$ and $A^2$) and the observable task feature ($Y$) are the same, but different choices of the third factor and hidden feature vector give the same response. Moreover, for the given choices of $A^3$ and $B^3$ it can be easily verified that $\norm{A^3 - B^3}_F = O(r)$ and $\sin \theta (A^3, B^3) = \norm{{A^3}^\top_\perp  B^3}_\op = \frac{1}{\sqrt{2}} = \sin (\pi/4)$. Therefore, even if we exactly know the factors $A^1$ and $A^2$, it is impossible to approximate $A^3$ either in terms of Frobenius norm or in terms of $\sin \theta$ distance.
\end{proof}
\section{Proof of Theorem \ref{thm:tensor_regression_bound} }

Our analysis builds upon the work by \citet{TSHK11}, who analyzed the performance of tensor regression with overlapped Schatten-1 norm. Recall the definition of the term $D_1 = \sqrt{d_1} +\sqrt{d_2} + \sqrt{T} +\sqrt{d_1d_2} + \sqrt{d_1 T} +\sqrt{d_2 T}$. \citet{TSHK11} showed that when (1) the true tensor $B$ has multi-way rank bounded by $r$, i.e. $\max\set{\rank(B_{(1)}), \rank(B_{(2)}), \rank(B_{(3)}} \le r$, (2) the number of samples $N \ge c_1 r D_1^2$,and (3) the covariate tensors $X_i$ are drawn iid from standard Gaussian distribution, then choosing $\lambda \ge c_2 \frac{\sigma D_1}{\sqrt{N}}$ guarantees the following:
\begin{equation}\label{eq:tensor-regression-result}
    \norm{B - \hat{B}}_F \le O\left(\frac{\sigma \sqrt{r} D_1}{\sqrt{N}} \right) 
\end{equation}
with high probability. 
In order to state the main ideas behind the proof and how they can be adapted for our setting, we  introduce the following notations. 
\begin{itemize}
    \item $\Xf : \R^{d_1 \times d_2 \times T} \rightarrow \R^N$ defined as $\Xf(W)_i = \left \langle \calX_i, W \right \rangle$.
    \item Adjoint operator $\Xf^*: \R^N \rightarrow \R^{d_1 \times d_2 \times T}$ defined as $\Xf(\overrightarrow{\veps}) = \sum_{i=1}^N \veps_i \calX_i$.
    \item Given a tensor $\Delta \in \R^{d_1 \times d_2 \times T}$ write its $k$-th mode as $\Delta_{(k)}$ as $\Delta_{(k)} = \Delta'_{(k)} + \Delta^{''}_{(k)}$ where the row and column space of $\Delta^{'}_{(k)}$ are orthogonal to the row and column spaces of $B_{(k)}$ respectively. 
    \item A constraint set $\calC = \set{\Delta \in \R^{d_1 \times d_2 \times T}: (1) \textrm{rank}(\Delta'_{(k)}) \le 2r \ \forall k \ \textrm{and} \ (2) \sum_k \norm{\Delta''_{(k)}}_{\star} \le 3 \sum_k \norm{\Delta'_{(k)}}_{\star}}$.
\end{itemize}

\begin{definition}[Restricted Strong Convexity]
There exists a constant $\kappa(\Xf)$ such that for all tensors in $\Delta \in \calC$, we have
\begin{equation*}
    \frac{\norm{\Xf(\Delta)}_2^2 }{N} \ge \kappa(\Xf) \norm{\Delta}_F^2.
\end{equation*}
\end{definition}

With this definition, \citet{TSHK11} proves the guarantee in eq. \ref{eq:tensor-regression-result} in three steps. 
\begin{enumerate}
    \item If the restricted strong convexity is satisfied with a constant $\kappa(\Xf)$ and  $\lambda$ is chosen to be at least $\frac{2}{N} \norm{\Xf^*(\overrightarrow{\veps})}_{\mean}$\footnote{$\norm{\cdot}_{\mean}$ is the dual norm of $\norm{\cdot}_S$ and is defined as $\norm{A}_{\mean} = 1/3 \sum_{k=1}^3 \norm{W_{(k)} }_{\textrm{op}}$}, then we have the following guarantee:
\begin{equation}\label{eq:tensor-regression-tjj}
\norm{B - \hat{B}}_F \le O\left(\frac{\lambda \sqrt{r} }{\kappa(\Xf) } \right).
\end{equation}
\item Gaussian design (i.e. $\calX_i \sim \Normal(0,I_{d_1\times d_2\times T})$ satisfies restricted strong convexity with constant $\kappa(\Xf) = O(1)$.
\item Additionally, Gaussian design satisfies $\norm{\Xf^*(\overrightarrow{\veps})}_{\mean} = O(\sigma D_1 \sqrt{N})$ with high probability.
\end{enumerate}

 We now carry out these steps for our setting. First, lemma \ref{lem:rsc-gaussian-design} proves that our setting satisfies \emph{restricted strong convexity} with high probability.
 As a result of this lemma, we see that our setting satisfies  restricted strong convexity with constant $\kappa(\Xf) = \frac{\lambda_{\min}(\Sigma_y) \lambda_{\max}(\Sigma) }{36 T}.$ Compared to \cite{TSHK11}, we don't get a constant independent of the number of tasks $T$ and it gets worse with increasing $T$. 
The constant is $O(1/T)$ because of uniform sampling, where each individual samples one task uniformly at random out of $T$ tasks. For other assignment scheme, the constant could be adjusted appropriately.

Recall, that we need to choose $\lambda > \frac{2}{N} \norm{\Xf^*(\overrightarrow{\veps})}_{\mean}$. Lemma ~\ref{lem:bound-mean-norm} lemma provides a  lower bound of $O(\sigma D_1 / \sqrt{N})$ on $\norm{\Xf^*(\overrightarrow{\veps})}_{\mean}$.
Now we substitute, $\lambda = O\left(\frac{\sigma D_1}{\sqrt{N}}\right)$ and $\kappa(\Xf) = \left(\frac{\lambda_{\min}(\Sigma_y) \lambda_{\max}(\Sigma) }{ T}\right)$ in equation \ref{eq:tensor-regression-tjj} to get the main result for our setting. 
If we fix $d_1$ and $d_2$, then the bound scales as $\frac{T^{3/2}}{\sqrt{N}}$. This is worse by a factor of $\sqrt{T}$ compared to the result of \cite{TSHK11}. Because of uniform sampling the number of effective samples is $\sqrt{N/T}$, and one should expect a bound of $\frac{\sqrt{T}}{\sqrt{N/T}} = \frac{T}{\sqrt{N}}$.

\begin{lemma}\label{lem:rsc-gaussian-design}
Suppose $X_1,\ldots,X_N \sim_{\textrm{iid}} \Normal(0,\Sigma)$, $Y_1, \ldots, Y_T \simiid \Normal(0,\Sigma_y)$, and $t(i) \sim \text{Unif}\set{1,\ldots,T}$ for each $i$. If $N \ge O(r D_1^2 \lambda_{\max}(\Sigma) / \lambda_{\min}(\Sigma_y))$, then for any $\Delta \in \calC$, the following holds
$$
\frac{\norm{\Xf(\Delta)}_2}{\sqrt{N}} \ge \frac{\sqrt{\lambda_{\min}(\Sigma_y)\lambda_{\min}(\Sigma)} }{6\sqrt{T}} \norm{\Delta}_F
$$
with probability at least $1 - e^{-\Omega(N/T)}$.
\end{lemma}

\begin{proof}
We first assume $X_1,\ldots,X_N \simiid \Normal(0,\Identity)$ and derive our result. We will then see how a standard trick handles the case of general covariance matrix.

Since $\Sigma_y$ is  a positive-definite matrix, we can right its eigen-decomposition as $\Sigma_y = U^\top D U$ where $U \in \in \R^{d_2 \times d_2}$ is an orthonormal matrix. This implies that there exists a matrix $M = UD^{1/2}$ such that $\Sigma_y = M^\top M$. Moreover the columns of $M$ form an orthogonal basis of $\R^{d_2}$ and $L_2$ norm of any column of $M$ is at least $\lambda^{1/2}_{\min}(\Sigma_y)$. 
Given a tensor $\Delta \in \calC$ let us define a new tensor $\Delta_M \in \R^{d_1\times d_2 \times T}$ defined as $\Delta_M(a,b,t) =\Delta_{a:t}^\top M_b$. We first prove the following result.
\begin{align}\label{eq:temp-result}
    \frac{\norm{\Xf(\Delta)}_2}{\sqrt{N}} \ge \frac{\norm{\Delta_M}_F}{4\sqrt{T}} - \frac{D_1}{3\sqrt{TN}} \norm{\Delta_M}_{S} 
\end{align}

 We can assume that $\norm{\Delta_M}_F = 1$. Otherwise, we construct a new tensor $\tilde{\Delta} = {\Delta} / {\norm{\Delta_M}_F}$, and the new tensor has $\norm{\tilde{\Delta}_M}_S = 1/3 \sum_k \norm{\tilde{\Delta}_{M(k)}}_\star = 1/(3\norm{\Delta_M}_F) \sum_k \norm{{\Delta}_{(k)}}_\star = \norm{\Delta}_S / \norm{\Delta_M}_F$, and the claim is valid upto rescaling by ${\norm{\Delta}_F}$. We now proceed similar to the proof of proposition 1 in \cite{NW11}. First, by a peeling argument very similar to the proof of proposition 1 in \cite{NW11}, it is enough to consider the case $\norm{\Delta_M}_S \le t$ and show the following:
 $$
 \frac{\norm{\Xf(\Delta)}_2}{\sqrt{N}} \ge \frac{1}{4\sqrt{T}} - \frac{tD_1}{\sqrt{TN}} 
 $$
 for all tensors $\Delta$ in the set $\calR(t) = \set{\Gamma  \in \R^{d_1 \times d_2 \times T}: \norm{\Gamma_M}_F = 1 \ \textrm{and} \ \norm{\Gamma_M}_S \le t}$. Let $S^{N-1} = \set{u \in \R^N: \norm{u}_2 = 1}$ and for all $u \in S^{N-1}$ we define $Z_{u,\Delta} = \left \langle u, \Xf(\Delta) \right \rangle$ for any $\Delta \in \R^{d_1 \times d_2 \times T}$. Note that,
 $$
 Z_{u,\Delta} = \sum_{i=1}^N u_i \left \langle \calX_i, \Delta \right \rangle = \sum_{i=1}^N u_i \left \langle X_iY_{t(i)}^\top, \Delta_{::t(i)} \right \rangle.
 $$
 Moreover,
 \begin{align*}
     \E\left[(Z_{u,\Delta} - Z_{u',\Delta'})^2\right] &= \frac{1}{T}  \sum_{i,a,t} \E\left[ \left\{ \sum_b Y_t(b) (u_i \Delta(a,b,t) - u_i' \Delta'(a,b,t))\right\}^2|Y_t\right] \\
     &= \frac{1}{T} \sum_{i,a,t} \sum_b \Sigma_y(b,b)(u_i \Delta(a,b,t) - u_i' \Delta'(a,b,t))^2  \\ &+ \frac{1}{T}\sum_{i,a,t} \sum_{b\neq b'} \Sigma_y(b,b') (u_i \Delta(a,b,t) - u_i' \Delta'(a,b,t))^2 (u_i \Delta(a,b',t) - u_i' \Delta'(a,b',t))^2  \\
 \end{align*}
 
 We now use the eigen-decomposition of $\Sigma_y = M^\top M$ to get the following result.
 \begin{align*}
     \E\left[(Z_{u,\Delta} - Z_{u',\Delta'})^2\right] &= \frac{1}{T} \sum_{i,a,t} \norm{u_i \Delta_{a:t}M - u_i' \Delta'_{a:t}M}_2^2 \\
     &= \frac{1}{T} \sum_{i,a,t,b} \left(u_i \Delta_{a:t}^{\top} M_b - u'_i {\Delta'_{a:t}}^{\top} M_b \right)^2 \\
     &= \frac{1}{T} \norm{u \otimes \Delta_M - u' \otimes \Delta'_M}_F^2
 \end{align*}
 where in the last line we write $\Delta_M$ for the tensor $\Delta_M(a,b,t) = \Delta_{a:t}^\top M_b$. We now consider a second mean-zero gaussian process $W_{u,\Delta} = \frac{1}{\sqrt{T}} \left ( \left \langle g, u  \right \rangle + \left \langle G, \Delta_M  \right \rangle \right)$, where $g \in \R^N$ and $G \in \R^{d_1 \times d_2 \times T}$ are iid with $N(0,1)$ entries. We have
 \begin{align*}
     \E\left[(W_{u,\Delta} - W_{u',\Delta'})^2\right] = \frac{1}{T}\norm{u - u'}_2^2 + \frac{1}{T}\norm{ {\Delta_M} - {\Delta}'_M}_F^2.
 \end{align*}
 We now verify that the two gaussian processes $(Z_{u,\Delta})$ and $(W_{u,\Delta})$ satisfy the requried conditions of Gordon-Slepian's inquaility (lemma~\ref{lem:gordon}).
 We always have the following inequality  $\norm{u \otimes \Delta_M - u' \otimes \Delta'_M}_F^2 \le \norm{u - u'}_2^2 + \norm{ {\Delta_M} - {\Delta}'_M}_F^2$ for all pairs $(u,\Delta)$ and $(u',\Delta')$. Moreover,  if $\Delta = \Delta'$, then $\Delta_M = \Delta'_M$ and equality holds. 
 
 Therefore, the two required conditions of Gordon-Slepian inequality(\cref{lem:gordon}) are satisfied for the gaussian process $(W_{\Delta,u})_{\Delta \in \calR(t), u \in S^{N-1}}$ and $(Z_{\Delta,u})_{\Delta \in \calR(t), u \in S^{N-1}}$ we get the following inequality:
 \begin{align*}
     \E \inf_{\Delta \in \calR(t)} \sup_{u \in S^{N-1}} W_{\Delta, u} \le \E \inf_{\Delta \in \calR(t)} \sup_{u \in S^{N-1}} Z_{\Delta, u} 
 \end{align*}
 which helps us bound $\inf_{\Delta \in \calR(t)} \norm{\Xf(\Delta)}_2$.
 \begin{align*}
 \E\left[\inf_{\Delta \in \calR(t)} \norm{\Xf(\Delta)}_2\right] &= \E\left[\inf_{\Delta \in \calR(t)} \sup_{u \in S^{N-1}} Z_{u,\Delta} \right] \ge \E\left[\inf_{\Delta \in \calR(t)} \sup_{u \in S^{N-1}} W_{u,\Delta} \right] \\
 &= \E\left[ \sup_{u \in S^{N-1} }\frac{1}{\sqrt{T}} \left \langle g, u  \right \rangle \right]   + \E\left[ \inf_{\Delta \in \calR(t)} \frac{1}{\sqrt{T}}\left \langle G, \Delta_M  \right \rangle \right] \\
 &= \frac{1}{\sqrt{T}} \E \left[\norm{g}_2 \right] - \frac{1}{\sqrt{T}} \E\left[ \sup_{\Delta \in \calR(t)} \langle G, \Delta_M \rangle \right] \\
 &\ge \frac{\sqrt{N}}{2\sqrt{T}} - \frac{t}{\sqrt{T}} \E\left[ \norm{G}_{\textrm{mean}}\right]
 \end{align*}
 Here the last inequality uses $\langle G, \Delta_M \rangle \le \norm{G}_{\textrm{mean}} \norm{\Delta_M}_S \le t  \norm{G}_{\textrm{mean}}$. Moreover, for a random gaussian matrix of dimension $m_1 \times m_2$ the expected value of its operator norm is bounded by $\sqrt{m_1} + \sqrt{m_2}$. This gives us $\E\left[ \norm{G}_{\textrm{mean}}\right] = \frac{1}{3} \sum_k \E\left[ \norm{G_(k)}_{\textrm{op}}\right] = D_1/3$.

 \begin{align*}
     \frac{\E\left[\inf_{\Delta \in \calR(t)} \norm{\Xf(\Delta)}_2\right]}{\sqrt{N}} \ge \frac{1}{2\sqrt{T}} - \frac{tD_1}{3\sqrt{TN}}
 \end{align*}
 Now the function $f(\{X_i\}_{i \in [N]}) = \inf_{\Delta \in \calR(t)} \frac{ \norm{\Xf(\Delta)}_2}{\sqrt{N}}$ is $1/\sqrt{N}$-Lipschitz. Therefore for all $\delta > 0$, we have
 \begin{align*}
P\left(\inf_{\Delta \in \calR(t)} \frac{ \norm{\Xf(\Delta)}_2}{\sqrt{N}} \le  \frac{1}{2\sqrt{T}} - \frac{tD_1}{3\sqrt{TN}} - \delta \right) \le 2 \exp\left( -\frac{\delta^2 N}{2}\right) 
 \end{align*}
 Now substituting $\delta = 1/(4\sqrt{T})$ we get that the identity defined in eq. \ref{eq:temp-result} holds. We now relate the norms of $\Delta_M$ and $\Delta$. Let $k=1$ and $\Delta_{(1)} = U_1 D_1 V_1^\top$ be the corresponding singular value decomposition. Then $\norm{\Delta_{(1)}}_\star = \trace(D_1)$. If we define $\tilde{V}_{1}$ a new matrix with $s$-th column $\tilde{v}_{1,s}(b,t) = \sum_{b'} v_{1,s}(b',t)M(b',b)$, then we have $\Delta_{M,(1)} = U_1 D_1 \tilde{V}_1^\top$. This implies that $\norm{\Delta_{(1)}}_\star = \norm{\Delta_{M,(1)}}_\star$. Similarly, it can be shown that  $\norm{\Delta_{(2)}}_\star = \norm{\Delta_{M,(2)}}_\star$ and $\norm{\Delta_{(3)}}_\star = \norm{\Delta_{M,(3)}}_\star$. This implies that $\norm{\Delta}_S = \norm{\Delta_M}_S$. For the Frobenius norm we use the fact that the columns of $M_b$ form an orthogonal basis of $\R^{d_2}$ and get $\norm{\Delta_M}_F^2 = \sum_{a,b,t} (\Delta_{a:t}^\top M_b)^2 \ge \lambda_{\min}(\Sigma_y) \sum_{a,t} \norm{\Delta_{a:t}}_2^2 = \lambda_{\min}(\Sigma_y) \norm{\Delta}_F^2$. The previous two relations give us the following bound.
 \begin{align*}
    \frac{\norm{\Xf(\Delta)}_2}{\sqrt{N}} &\ge \frac{\norm{\Delta_M}_F}{4\sqrt{T}} - \frac{D_1}{3\sqrt{TN}} \norm{\Delta_M}_{S} 
    \ge \frac{\lambda^{1/2}_{\min}(\Sigma_y) \norm{\Delta}_F}{4\sqrt{T}} - \frac{D_1}{3\sqrt{TN}} \norm{\Delta}_{S}
\end{align*}
On the other hand, from the definition of the constraint set $\mathcal{C}$ we get $\norm{\Delta}_S = \frac{1}{3}\sum_k \norm{\Delta_{(k)}}_\star \le \frac{2}{3} \sum_k \norm{\Delta'_{(k)}}_\star \le  \frac{2}{3} \sqrt{2r} \sum_k \norm{\Delta'_{(k)}}_F \le \frac{2}{3} \sqrt{2r} \sum_k \norm{\Delta_{(k)}}_F = \sqrt{2r} \norm{\Delta}_F$. Therefore we have,
\begin{align*}
    \frac{\norm{\Xf(\Delta)}_2}{\sqrt{N}} &\ge\frac{\lambda^{1/2}_{\min}(\Sigma_y) \norm{\Delta}_F}{4\sqrt{T}} - \frac{D_1 \sqrt{2r}}{3\sqrt{TN}} \norm{\Delta}_{F} \ge \frac{\lambda^{1/2}_{\min}(\Sigma_y)}{6\sqrt{T}} \norm{\Delta}_F
\end{align*}
as long as $N \ge O(r D_1^2 / \lambda_{\min}(\Sigma_y)$.

Finally, we consider the case when $X_1,\ldots,X_N \simiid \Normal(0,\Sigma)$ for a general covariance matrix $\Sigma$. We define the following operator $T_\Sigma : \R^{d_1 \times d_2 \times T} \rightarrow \R^{d_1 \times d_2 \times T}$ defined as $T_\Sigma(\Delta)_{(1)} = \sqrt{\Sigma} \Delta_{(1)}$. We also define a gaussian random operator $\Xf' : \R^{d_1\times d_2 \times T} \rightarrow \R^N$ defined as $\Xf'_i = \left \langle \calX_i', T_\Sigma(\Delta) \right \rangle$. Here for each $i$, we define $\calX_i'$ as:
\begin{align*}
\calX_i'(\cdot, \cdot, t) =  \left\{\begin{array}{cc}
   \Sigma^{-1/2} X_i  & \textrm{ if } t(i) = t \\
    0 & \textrm{ o.w. }
\end{array}\right.
\end{align*}
Since each $\Sigma^{-1/2} X_i$ is drawn from standard gaussian distribution, we have
\begin{align*}
    \frac{\norm{\Xf'(\Delta)}_2}{\sqrt{N}}  \ge \frac{\lambda^{1/2}_{\min}(\Sigma_y)}{6\sqrt{T}} \norm{T_\Sigma(\Delta)}_F
\end{align*}
as long as $N \ge O(r D_1^2 \lambda_{\max}(\Sigma) / \lambda_{\min}(\Sigma_y)$. In deriving the above result, we use the inequality $\norm{T_\Sigma(\Delta)}_S \le \lambda^{1/2}_{\max}(\Sigma)\norm{\Delta}_S$.
Now, from the definition $\Xf'(\Delta)_i =  \left \langle \calX_i', T_\Sigma(\Delta) \right \rangle = \left \langle \calX_i, \Delta \right \rangle = \Xf(\Delta)_i$. Moreover, $\norm{T_\Sigma(\Delta)}_F = \norm{\sqrt{\Sigma} \Delta_{(1)}}_F \ge \lambda^{1/2}_{\min}(\Sigma)\norm{\Delta_{(1)}}_F = \lambda^{1/2}_{\min}(\Sigma)\norm{\Delta}_F$. Substituting this bound on the Frobenius norm gives us the desired result.
\end{proof}

\begin{lemma}\label{lem:bound-mean-norm}
$$P\left(\norm{\Xf^*(\overrightarrow{\veps})}_{\textrm{mean}} \le 20 \sigma \sqrt{N}D_1\right) \ge 1 - 2 e^{-\Omega(D_1^2)}.$$
\end{lemma}

\begin{proof}
As $\veps_1,\ldots,\veps_N$ are iid drawn from $N(0,\sigma^2)$ and the Euclidean norm $\norm{\overrightarrow{\veps}}_2$ is $1$-Lipschitz we get,
$$
P\left(\abs{\norm{\overrightarrow{\veps}}_2 - \E \norm{\overrightarrow{\veps}}_2} > \sigma \delta \right) \le 2 \exp\left( -\delta^2/2\right)
$$
Substituting $\delta = \sqrt{N}$ and observing that $\E \norm{\overrightarrow{\veps}}_2 \le 4\sigma \sqrt{N}$, we get that with probability at least $1 - \exp(-\Omega(N))$, $ \norm{\overrightarrow{\veps}}_2$ is bounded by $5\sigma \sqrt{N}$. We will write $\calE$ to denote this event.

\begin{align*}
\norm{\Xf^*(\overrightarrow{\veps})}_{\textrm{mean}} = \frac{1}{3} \sum_{k=1}^3 \norm{\Xf^*(\overrightarrow{\veps})_{(k)}}_{\textrm{op}} 
\end{align*}
We now bound the operator norm of each of the three modes of $\Xf^*(\overrightarrow{\veps})$ separately. Our proof follows the main ideas of the proof of Corollary 10.10 of \cite{W19}. Since $\Xf^*(\overrightarrow{\veps})_{(1)} \in \R^{d_1 \times d_2 T}$, we choose $1/4$-cover $\set{u^1,\ldots,u^{M_1}}$ of the set $S^{d_1 - 1} = \set{u \in \R^{d_1}: \norm{u}_2 = 1}$, and $1/4$-cover $\set{v^1,\ldots,v^{M_2}}$ of the set $S^{d_2T - 1} = \set{v \in \R^{d_2T}: \norm{v}_2 = 1}$. Note that, we can always choose the covers so that $M_1 \le 9^{d_1}$ and $M_2 \le 9^{d_2T}$.

\begin{equation*}
    \norm{\Xf^*(\overrightarrow{\veps})_{(1)}}_{\textrm{op}}  = \sup_{v \in S^{d_2T - 1}} \norm{\Xf^*(\overrightarrow{\veps})_{(1)}v}_2 \le \frac{1}{4} \norm{\Xf^*(\overrightarrow{\veps})_{(1)}}_{\textrm{op}} + \max_{l \in [M_2]} \norm{\Xf^*(\overrightarrow{\veps})_{(1)}v^l}_2
\end{equation*}
Similarly one can show that
\begin{equation*}
    \norm{\Xf^*(\overrightarrow{\veps})_{(1)}v^l}_2 \le \frac{1}{4} \norm{\Xf^*(\overrightarrow{\veps})_{(1)}}_{\textrm{op}} + \max_{j \in [M_1]} \left \langle u^j,  \Xf^*(\overrightarrow{\veps})_{(1)}v^l \right \rangle
\end{equation*}
This establishes the following bound on the operator norm in terms of the covers.
\begin{equation*}
  \norm{\Xf^*(\overrightarrow{\veps})_{(1)}}_{\textrm{op}} \le 2 \max_{j \in [M_1], l \in [M_2]} \abs{Z^{jl}} \quad \textrm{where } Z^{jl} = \left \langle u^j,  \Xf^*(\overrightarrow{\veps})_{(1)}v^l \right \rangle
\end{equation*}
Using the definition of $\Xf^*(\overrightarrow{\veps})$, we get
\begin{align}
    Z^{jl} &= \sum_{i=1}^N \veps_i \left \langle u^j, \calX_{i,(1)} v^l \right \rangle 
    = \sum_{i=1}^N \veps_i \sum_{a,b} X_i(a) Y_{t(i)}(b) v^l(b,t(i)) u^j(a)
\end{align}
Since each entry of $X_i$ is drawn iid from $N(0,1)$, $Z^{jl}$ is a zero mean gaussian random variable with variance
\begin{equation*}
\sum_{i=1}^N \veps_i^2 \sum_{a} \{u^j(a)\}^2 \left( \sum_b  v^l(b,t(i)) Y_{t(i)}(b) \right)^2 \le \sum_{i=1}^N \veps_i^2 \sum_{a} \{u^j(a)\}^2 \sum_{b_1} Y^2_{t(i)}(b_1) \sum_{b_2} \{v^l(b_2,t(i))\}^2 \le \sum_{i=1}^N \veps_i^2
\end{equation*}
The last inequality uses -- the observed task features are normalized, $u \in S^{d_1 - 1}$ and $v \in S^{d_2 T - 1}$. Conditioned on the event the variance of each $Z^{jl}$ is bounded by $5\sigma \sqrt{N}$. Now we can provide a high probability bound on the operator norm.
\begin{align*}
    P\left(\norm{\Xf^*(\overrightarrow{\veps})_{(1)}}_{\textrm{op}} \ge T_N \right) &\le  P\left(2 \max_{j \in [M_1], l \in [M_2]} \abs{Z^{jl}} \ge T_N \right) \\
    &\le \sum_{j \in [M_1]} \sum_{j \in [M_2]} P\left(\abs{Z^{jl}} \ge T_N / 2\right) \\
    &\le 2 M_1 M_2 \exp\left\{-\frac{T_N^2}{50 \sigma^2 N}\right\} \le 2 \exp\left\{-\frac{T_N^2}{50 \sigma^2 N} + (d_1 + d_2 T) \log 9 \right\}
\end{align*}
If we choose $T_N \ge 20 \sigma \sqrt{N}D_1$, we get
\begin{align*}
    P\left(\norm{\Xf^*(\overrightarrow{\veps})_{(1)}}_{\textrm{op}} \ge  20 \sigma \sqrt{N}D_1\right) \le 2 \exp\left\{ -2D_1^2\right\}
\end{align*}
By a similar argument, we can bound the operator norm of the other two modes of $\Xf^*(\overrightarrow{\veps})$. 

\end{proof}

\begin{lemma}[Gordon's Inequality]\label{lem:gordon}
Let $(X_{ut})_{u \in U, t \in T}$ and $(Y_{ut})_{u \in U, t \in T}$ be two mean zero Gaussian processes indexed by pairs of points $(u,t)$ in a product space $U \times T$. Assume that we have
\begin{enumerate}
\item $\E(X_{ut} - X_{us})^2 \le \E(Y_{ut} - Y_{us})^2$ for all $u,t,s$.
\item $\E(X_{ut} - X_{vs})^2 \ge \E(Y_{ut} - Y_{vs})^2$ for all $u \neq v$ and $t,s$.
\end{enumerate}
Then we have
\[
\E \inf_{u \in U} \sup_{t \in T} X_{ut} \le \E \inf_{u \in U} \sup_{t \in T} Y_{ut}
\]
\end{lemma}
\begin{proof}
 See \cite{Ledoux13}, chapter 3.
\end{proof}

\section{Formal Statement and Proof of Lemma \ref{lem:tensor_decomp_guarantees}}
First, we state weaker set of assumptions under which the bounds of lemma \ref{lem:tensor_decomp_guarantees} holds.
We will make the following assumptions about the underlying tensor $A = [\Identity_r; A^1, A^2, A^3]$.
\begin{enumerate}[label=({A}{{\arabic*}})]
    \item  The columns of the factors of $A$ are orthogonal i.e.  $\left \langle A^1_i, A^1_j \right \rangle = \left \langle A^2_i, A^2_j \right \rangle = \left \langle A^3_i, A^3_j \right \rangle = 0$ for all $i \neq j$.
    \item The components have bounded norm i.e. 
    $\exists p < 3$,  $\max\set{\norm{A^{1^\top}}_{2 \rightarrow p}, \norm{A^{2^\top}}_{2 \rightarrow p}, \norm{A^{3^\top}}_{2 \rightarrow p}} \le 1 + o(1)$. 
    \item Rank is bounded i.e. $r = o(d)$.
\end{enumerate}
Recall the definition of $\calZ$, the matrix of unobserved features.
\begin{equation}
  \calZ = \begin{bmatrix} Z_1 & \cdots & Z_T \end{bmatrix}^T \in \R^{T\times d_3}
\end{equation}
Let $\calZ(s)$ denote the $s$-th column of the matrix $\calZ$. We will make the following assumptions about $\calZ$.
\begin{enumerate}[label=({Z}{{\arabic*}})]
\item $\frac{1}{d_3^{0.5+\gamma}}\Identity_{d_3} \preccurlyeq \calZ^\top \calZ \preccurlyeq \frac{1}{\sqrt{d_3}}\Identity_{d_3}$ for some $\gamma > 0$.
\item $\kappa(\calZ^\top \calZ) = \frac{\lambda_{\max}(\calZ^\top \calZ)}{\lambda_{\min}(\calZ^\top \calZ)} \le 1 + O(\sqrt{r/d})$.
\end{enumerate}

\begin{lemma}\label{lem:tensor_decomp_guarantees_formal}
Suppose tensor $A$ has rank $r$ CP-decomposition $A=[\Identity_r; A^1, A^2, A^3]$ and satisfies the assumptions (A1)-(A3),  the matrix of unobserved features $\calZ$ satisfies assumptions (Z1)-(Z2), and $N = \Omega\left(\frac{\sigma^2 T^2  D_1^2 r }{\lambda^2_{\min}(\Sigma_y) \lambda^2_{\min}(\Sigma) } \min\left\{\frac{1}{36}, \frac{\log r}{d} \right\}\right)$. Then we have the following guarantees:
\begin{align*}
&\max\left\{ \norm{\widehat{A^1} - A^1}_F, \norm{\widehat{A^2} - A^2}_F\right\} \le \tilde{O}\left( \frac{\sigma T D_1 r}{\sqrt{\lambda_{\min}(\calZ^\top \calZ)} \lambda_{\min}(\Sigma_y)\lambda_{\min}(\Sigma) \sqrt{N}}\right), \\ 
&\norm{\widehat{\calZ A^3} - \calZ A^3}_F \le  \tilde{O}\left( \frac{\sigma \sqrt{\lambda_{\max}(\calZ^\top \calZ)} T D_1 r^{1.5}}{\sqrt{\lambda_{\min}(\calZ^\top \calZ) } \lambda_{\min}(\Sigma_y)\lambda_{\min}(\Sigma) \sqrt{ N}}\right) 
\end{align*}
\end{lemma}
\begin{proof}
We will be using the robust tensor decomposition algorithm proposed by \cite{AGM14}. We first review the necessary conditions and the guarantees of their main algorithm. We are given a tensor $\hat{S} = S + \Psi$ where $S \in \R^{d_1 \times d_2 \times T}$ has rank-$r$ decomposition $S = [W; S^1, S^2, S^3]$ and $\Psi$ is a noise tensor with spectral norm $\psi = \norm{\Psi}$. We will write the singular values as $w_1 \ge w_2 \ge \ldots \ge w_r > 0$ with $\gamma = w_1/w_r$. Let $d = \max\{d_1, d_2, T \}$.
Moreover, suppose the tensor $S$ satisfies the following conditions.
\begin{enumerate}[label=({S}{{\arabic*}})]
    \item The components are incoherent i.e. $\max_{i \neq j} \set{\abs{\left \langle s^1_i, s^1_j \right \rangle}, \abs{\left \langle s^2_i, s^2_j \right \rangle}, \abs{\left \langle s^3_i, s^3_j \right \rangle}} \le \frac{\textrm{polylog}(d)}{\sqrt{d}}$.
    \item The components have bounded norm i.e. $\max\set{\norm{S^{1} }_{\op}, \norm{S^{2} }_{\op}, \norm{S^{3} }_{\op}} \le 1 + O\left(\sqrt{r/d}\right)$ and for some $p < 3$,  $\max\set{\norm{S^{1^\top}}_{2 \rightarrow p}, \norm{S^{1^\top}}_{2 \rightarrow p}, \norm{S^{1^\top}}_{2 \rightarrow p}} \le 1 + o(1)$. \footnote{For a matrix $M \in \R^{m \times n}$, define $\norm{M}_{q\rightarrow p} = \sup_{\norm{u}_q = 1} \norm{Mu}_p$. }
    \item Rank is bounded i.e. $r = o(d^{1.5}/\textrm{polylog}(d))$.
    \item $\psi \le \min \left\{\frac{1}{6}, O\left(\sqrt{\frac{\log r}{d}}\right) \right\}$.
    \item Tensor norm of $S$ is bounded i.e. $\norm{S} \le O(w_1)$ and $\norm{\sum_{i \neq j} w_i \langle s^1_i, s^1_j \rangle \langle s^2_i, s^2_j \rangle s^3_j } \le \frac{w_1 \textrm{polylog}(d) \sqrt{r}}{d}$.
    \item The maximum ratio of the weights satisfy $\gamma = O\left(\min\left\{\sqrt{d}, d^{1.5}/r \right\} \right)$.
\end{enumerate}

When the underlying tensor $S$ satisfies the conditioned above, \cite{AGM14} proposed an algorithm that returns an estimate $[\widehat{W}; \widehat{S^1}, \widehat{S^2}, \widehat{S^3}]$ with the following guarantees:
$$\max\left\{ \norm{\widehat{S^1} - S^1}_F, \norm{\widehat{S^2} - S^2}_F,\norm{\widehat{S^3} - S^3}_F\right\} \le \tilde{O}\left( \frac{\sqrt{r}\psi}{w_r}\right) \ \textrm{and} \ \norm{\widehat{W} - W}_2 \le \tilde{O}(\sqrt{r}\psi)$$

Consider the tensor $B = A \times_3 \calZ$. We now check that the conditions (S1)-(S6) are also satisfied when we consider the tensor $B$. $B$ has the following rank $r$ CP-decomposition $B = [G^{-1}; A^1, A^2, \calZ A^3 G]$ where the $i$-th entry of the diagonal matrix $G$ is $G_i = 1/\norm{\calZ A^3_i}_2$. This means that the rank of $B$ is also $r$ and (S3) is satisfied. The singular values of $B$ are given by $\norm{\calZ A^3_i}_2$ for $i \in [r]$. As each column of $A^3$ is normalized, the following result holds for any $i$.
\[
\lambda_{\min}(\calZ^\top \calZ) \le \norm{\calZ A^3_i}_2^2 \le \lambda_{\max}(\calZ^\top \calZ)  
\]
Therefore, the maximum ratio of singular values of the tensor $B$ is bounded by $\sqrt{\lambda_{\max}(\calZ^\top \calZ) / \lambda_{\min}(\calZ^\top \calZ)}$ which is bounded by $\sqrt{d}$ and assumption (S6) is satisfied.

We will write $C$ to denote the matrix $\calZ A^3 G$. Note that the $i$-th column of $C$ is given as $\calZ A^3_i/\norm{\calZ A^3_i}_2$. In order to check condition (S1), we need to verify $\abs{\ip{C_i, C_j}} \le \frac{\textrm{polylog}(d)}{\sqrt{d}}$. Note that $\abs{\ip{C_i, C_j}} = \frac{\abs{\langle \calZ A^3_i, \calZ A^3_j \rangle}}{\norm{\calZ A^3_i}_2 \norm{\calZ A^3_j}_2} \le \frac{\abs{\langle \calZ A^3_i, \calZ A^3_j \rangle}}{\lambda_{\min}(\calZ^\top \calZ) }$. 
\begin{align*}
    \frac{1}{2}(A^3_i + A^3_j)^\top \calZ^\top \calZ \frac{1}{2}(A^3_i + A^3_j) = {A^3_i}^\top \calZ^\top \calZ A^3_i + {A^3_j}^\top \calZ^\top \calZ A^3_j + 2 {A^3_i}^\top \calZ^\top \calZ A^3_j
\end{align*}
Using assumption (Z2) we get,
\begin{align*}
    2 {A^3_i}^\top \calZ^\top \calZ A^3_j \le \frac{1}{\sqrt{d_3}} - {A^3_i}^\top \calZ^\top \calZ A^3_i - {A^3_j}^\top \calZ^\top \calZ A^3_j \le \frac{1}{\sqrt{d_3}} - \frac{2}{d_3^{0.5+\gamma}} = O\left(\frac{1}{\sqrt{d_3}} \right)
\end{align*}

In order to check (S2), notice that $\norm{B^1}_{\op} = \norm{A^1}_{\op} \le 1 + O\left(\sqrt{r/d} \right)$. Same result holds for $B^2$. For the third factor we have, $\norm{B^3}_{\op} = \norm{\calZ A^3 G}_{\op} \le \norm{\calZ}_{\op} \norm{A^3}_{\op} \max_i \frac{1}{\norm{\calZ A^3_i}_2} \le \sqrt{\frac{\lambda_{\max}(\calZ^\top \calZ)}{\lambda_{\min}(\calZ^\top \calZ)}} \norm{A^3}_{\op} \le \left(1 + O(\sqrt{r/d})\right)$. For the second part of (S2), we just need to bound $\norm{(\calZ A^3G)^\top}_{2 \rightarrow p}$.
\begin{align*}
\norm{(\calZ A^3G)^\top}_{2 \rightarrow p} &= \norm{\calZ A^3 G}_{\frac{p}{p-1} \rightarrow 2} \quad \textrm{[By lemma 8 of \cite{KMW18}}\\
&= \max_{x: \norm{x}_{p/(p-1)} = 1} \norm{\calZ A^3 G x}_2  \le \norm{\calZ}_{\textrm{op}} \max_{x: \norm{x}_{p/(p-1)} = 1} \norm{A^3 G x}_2  \\
&=\norm{\calZ}_{\textrm{op}} \norm{A^3 G}_{\frac{p}{p-1} \rightarrow 2} = \norm{\calZ}_{\textrm{op}} \norm{G {A^3}^\top }_{2 \rightarrow p} \\
&\le \norm{\calZ}_{\textrm{op}} \norm{G}_p \norm{{A^3}^\top }_{2 \rightarrow p} \le \sqrt{\frac{\lambda_{\max}(\calZ^\top \calZ)}{\lambda_{\min}(\calZ^\top \calZ)}} \norm{{A^3}^\top}_{2 \rightarrow p} \le 1 + o(1)
 \end{align*}
The last line uses (A2), (A3), and (Z2).

If we write $\hat{B} = B + \Psi$, from the guarantees of tensor regression (theorem~\ref{thm:tensor_regression_bound}) we have $\psi = \norm{\Psi} \le \norm{\Psi}_F \le O\left(\frac{\sigma T  D_1 \sqrt{r}}{\lambda_{\min}(\Sigma_y) \lambda_{\min}(\Sigma) \sqrt{N}} \right)$. So as long as, $N \ge O\left(\frac{\sigma^2 T^2  D_1^2 r }{\lambda_{\min}(\Sigma_y)^2\lambda_{\min}(\Sigma)^2} \min\left\{\frac{1}{36}, \frac{\log r}{d} \right\}\right)$,  condition (S4) is satisfied.

We now verify condition (S5). Fix three vectors $a \in \R^{d_1}, b \in \R^{d_2},$ and $c \in \R^T$ with $\norm{a}_2 = \norm{b}_2 = \norm{c}_2 = 1$.
\begin{align*}
B(x,y,z) &= \sum_{i=1}^r G^{-1}_i (A^{1^\top}a)_i (A^{2^\top}b)_i ((\calZ A G)^{\top}c)_i \\
&\le \max_i G^{-1}_i \norm{A^{1^\top}a}_3 \norm{A^{2^\top}b}_3 \norm{(\calZ AG)^{\top}c}_3 \\
&\le \max_i G^{-1}_i \norm{A^{1^\top}}_{2 \rightarrow 3} \norm{a}_2 \norm{A^{2^\top}}_{2 \rightarrow 3} \norm{b}_2 \norm{(\calZ A G)^{\top}}_{2 \rightarrow 3} \norm{c}_2 \\
&\le \max_i G^{-1}_i \norm{A^{1^\top}}_{2 \rightarrow p} \norm{A^{2^\top}}_{2 \rightarrow p} \norm{(\calZ A G)^{\top}}_{2 \rightarrow p}  = O( \max_i G^{-1}_i)
\end{align*}
The first inequality uses Corollary 3 from \cite{AGM14}, which applies H\"older's inequality three times. The inequality on the following fact. For any matrix $M$, $\norm{M}_{2 \rightarrow 3} \le \norm{M}_{2 \rightarrow p}$ which follows from the definition of $\norm{\cdot}_{2 \rightarrow p}$ and $p < 3$. Finally, the second part of condition (S5) follows immediately as the columns of $A^1$ and $A^2$ are orthonormal. 
%

Therefore, we conclude that the tensor $B = A \times_3 \calZ$ satisfies assumptions (S1)-(S6) and we can apply robust tensor decomposition algorithm from \cite{AGM14}. As we can write $B$ as $B = \llbracket G^{-1}; A^1, A^2, \calZ A^3 G\rrbracket$, we get the following guarantees.
\begin{align*}
&\max\left\{ \norm{\widehat{A^1} - A^1}_F, \norm{\widehat{A^2} - A^2}_F,\norm{\widehat{\calZ A^3 G} - \calZ A^3 G}_F\right\} \le \tilde{O}\left( \frac{\sigma T D_1 r}{\sqrt{\lambda_{\min}(\calZ^\top \calZ)} \lambda_{\min}(\Sigma_y) \lambda_{\min}(\Sigma) \sqrt{N}}\right) \ \textrm{ and } \\
&\norm{\widehat{G^{-1}} - G^{-1}}_2 \le \tilde{O}\left(\frac{\sigma T D_1 r}{\lambda_{\min}(\Sigma_y)\lambda_{\min}(\Sigma) \sqrt{N}}\right)
\end{align*}

Since we also have an estimate of $G^{-1}$ we can estimate $\calZ A^3$ by $\widehat{\calZ A^3 G} \widehat{G^{-1}}$. Then we have the following guarantee.
\begin{align*}
    \norm{\widehat{\calZ A^3} - \calZ A^3}_F &= \norm{\widehat{\calZ A^3 G} \widehat{G^{-1}} - \calZ A^3 G G^{-1}}_F \\
    &= \norm{\widehat{\calZ A^3 G} \widehat{G^{-1}} - {\calZ A^3 G} \widehat{G^{-1}} + {\calZ A^3 G} \widehat{G^{-1}} - \calZ A^3 G G^{-1}}_F\\
    &\le \norm{\widehat{\calZ A^3 G}  - {\calZ A^3 G} }_F \norm{\widehat{G^{-1}}}_F + \norm{\calZ A^3 G}_F \norm{\widehat{G^{-1}} - G^{-1}}_2\\
    &\le \norm{\widehat{\calZ A^3 G}  - {\calZ A^3 G} }_F \left(\norm{\widehat{G^{-1}} - G^{-1}}_F + \norm{G^{-1}}_F\right) + \sqrt{r} \norm{\calZ A^3 G}_{\op} \norm{\widehat{G^{-1}} - G^{-1}}_2\\
    &= \tilde{O}\left(\frac{\sqrt{\kappa(\calZ^\top \calZ)}\sigma  T D_1 r^{1.5}}{ \lambda_{\min}(\Sigma_y) \lambda_{\min}(\Sigma) \sqrt{N}} \right) 
\end{align*}

\end{proof}

\section{Formal Statement and Proof of Theorem \ref{thm:meta-test-1}}

\begin{theorem}\label{thm:meta-test-I-formal}
Each covariate vector $X_i$ is mean-zero, satisfies $\E[X_iX_i^\top] = \Sigma$ and $\Sigma$-sub-gaussian, and  $\max\left\{\norm{\hat{A^1} - A^1}_F, \norm{\hat{A^2} - A^2}_F \right\} \le \delta$. Additionally, suppose that  $N_2 \ge O\left(r \left(\norm{Y_0}_2^2 \frac{\lambda_{\max}(\Sigma)}{\lambda_{\min}(\Sigma)}\right)^2 \log(2/\delta_1)\right)$, and $\abs{Y_0^\top \hat{A^2}_i} \ge \eta \norm{Y_0}_2$ for all $i \in [r]$. Then with probability at least $1-\delta_1$ we have
$$\E_{X_0}\left[ \left(A(X_0, Y_0, Z_0) - \hat{A}(X_0, Y_0, \widehat{Z}_0 \right)^2\right] = O\left(\frac{B_1}{\eta^2} r^2 \delta^2 + \frac{B_2}{\eta^2} \frac{r^2}{N_2} \right),$$
for $B_1 = \frac{\lambda_{\max}(\Sigma)}{\lambda_{\min}(\Sigma)} \E[\norm{X_0}_2^2]\norm{Y_0}_2^2  \norm{Z_0}_2^2$ and $B_2 = \frac{ \E[\norm{X_0}_2^2] }{\lambda_{\min}(\Sigma)}$.
\end{theorem}
\begin{proof}
Mean squared error is given as
\begin{align}
    &\E_{X_0}\left[ \left( \hat{A}(X_0, Y_0, \hat{Z}_0) - A(X_0, Y_0, Z_0)\right)^2\right] \nonumber
   \\ &= \E_{X_0}\left[\left(  (Y_0^\top \hat{A^2} \odot X_0^\top  \hat{A^1}) \widehat{{A^3}^\top Z_0} - (Y_0^\top A^2 \odot X_0^\top  A^1){A^3}^\top Z_0\right)^2 \right]\label{eq:inter_meta_test}
\end{align}
We will write $u \in \R^r$ to denote the vector $(Y_0^\top A^2 \odot X_0 A^1)$ and $\hat{u}$ to denote its estimate $(Y_0^\top \hat{A^2} \odot X_0 \hat{A^1})$. 
\begin{align*}
&\E_{X_0}\left[\left(  \hat{u}^\top \widehat{{A^3}^\top Z_0} - u^\top {A^3}^\top Z_0\right)^2 \right] = E_{X_0}\left[\left(  (\hat{u} - u)^\top \widehat{{A^3}^\top Z_0}+ u^\top (\widehat{{A^3}^\top Z_0} - {A^3}^\top Z_0)\right)^2 \right] \\
&\le 2 \E_{X_0} [\norm{\hat{u} - u}_2^2] \norm{\widehat{{A^3}^\top Z_0}}_2^2 + 2 \E_{X_0}[\norm{u}_2^2]  \norm{\widehat{{A^3}^\top Z_0} - {A^3}^\top Z_0}_2^2 \\
\end{align*}
Now, $\norm{u}_2^2 = \sum_{i=1}^r (Y_0^\top A^2_i)^2 (X_0^\top A^1_i)^2 \le \sum_{i=1}^r \norm{Y_0}_2^2 \norm{A^2_i}_2^2 \norm{X_0}_2^2 \norm{A^1_i}_2^2 = r \norm{Y_0}_2^2 \norm{X_0}_2^2$. As $X_0$ is drawn from a zero-mean, $\Sigma$-subgaussian distribution, we have $\E[\norm{u}_2^2] = O(\norm{Y_0}_2^2  r \E[\norm{X_0}_2^2])$. Moreover,
\begin{align*}
\norm{\hat{u} - u}_2^2 &= \sum_{i=1}^r \left[(Y_0^\top A^2_i) (X_0^\top A^1_i) - (Y_0^\top \hat{A^2_i}) (X_0^\top \hat{A^1_i})  \right]^2\\
&= \sum_{i=1}^r \left[Y_0^\top A^2_i (X_0^\top A^1_i - X_0^\top \hat{A^1_i})  + X_0^\top \hat{A^1_i}(Y_0^\top A^2_i - Y_0^\top \hat{A^2_i})\right]^2 \\
&\le 2 \sum_{i=1}^r (Y_0^\top A^2_i)^2 (X_0^\top A^1_i - X_0^\top \hat{A^1_i})^2 + 2 \sum_{i=1}^r (X_0^\top \hat{A^1_i})^2(Y_0^\top A^2_i - Y_0^\top \hat{A^2_i})^2 \\
&\le 2 \sum_{i=1}^r \norm{Y_0}_2^2 \norm{A^2_i}_2^2 \norm{X_0}_2^2 \norm{A^1_i \hat{A^1_i}}_2^2 + 2 \sum_{i=1}^r \norm{X_0}_2^2 \norm{\hat{A^1_i}}_2^2 \norm{Y_0}_2^2 \norm{\hat{A^2_i} - A^2_i}_2^2 \\
&= 2 \norm{X_0}_2^2 \norm{Y_0}_2^2 \left(\norm{A^1 - \hat{A^1}}_F^2 + \norm{A^2 - \hat{A^2}}_F^2\right) \\
&\le 4 \norm{X_0}_2^2 \norm{Y_0}_2^2{\delta^2}
\end{align*}
Therefore, $\E[\norm{\hat{u} - u}_2^2] = O(\norm{Y_0}_2^2 \E[\norm{X_0}_2^2] \delta^2 )$.  

This gives us a bound of 
\begin{equation}\label{eq:mse-upper-bound}
O\left(\norm{Y_0}_2^2  \E[\norm{X_0}_2^2] \left({\delta^2} \norm{\widehat{{A^3}^\top Z_0}}_2^2 + r\norm{\widehat{{A^3}^\top Z_0} - {A^3}^\top Z_0}_2^2 \right) \right)
\end{equation} 
on the mean-squared error. We first bound $\norm{\widehat{{A^3}^\top Z_0} - {A^3}^\top Z_0}_2^2$.
Recall that if we write $\VV = (Y_0^\top \hat{A^2} \odot \calX_0 \hat{A^1})$, then we can write $\widehat{{A^3}^\top Z_0}$ as $\left( \VV^\top  \VV  \right)^{-1} \VV^\top  R$. 

\begin{align*}
 \widehat{{A^3}^\top Z_0} - {A^3}^\top Z_0 &= \left( \VV^\top  \VV  \right)^{-1} \VV^\top  R  - {A^3}^\top Z_0 \\
&=  \left( \VV^\top  \VV  \right)^{-1}\VV^\top (V{A^3}^\top Z_0 +\bmeps)  - {A^3}^\top Z_0 \\
&= \left( \VV^\top  \VV  \right)^{-1}\VV^\top \bmeps \\&+ \left( \VV^\top  \VV  \right)^{-1} \VV^\top V{A^3}^\top Z_0  - {A^3}^\top Z_0 
\end{align*}
Lemmas \ref{lem:bias-meta-test-I} and \ref{lem:variance-meta-test-I} respectively bound the bias and the variance term. 
Substituting these bounds we get $\norm{\widehat{{A^3}^\top {Z}_0} - {A^3}^\top Z_0}_2^2 = O\left(\frac{C_1}{\eta^2} \frac{r}{ N_2} + \frac{C_2}{\eta^2} {r \delta^2} \right)$ for $C_1 = \frac{1}{\norm{Y_0}_2^2 \lambda_{\min}(\Sigma)}$ and $C_2 =  \frac{\lambda_{\max}(\Sigma) }{\lambda_{\min}(\Sigma)}\norm{Z_0}_2^2$. We now consider the remaining term $\norm{\hat{Z_0}}_2^2$ in the upper bound on MSE (\cref{eq:mse-upper-bound}).
\begin{align*}
&\norm{\widehat{{A^3}^\top {Z}_0}}_2^2 = \norm{\left(\VV^\top  \VV  \right)^{-1}\VV^\top  R}_2^2 =  \norm{\left( \VV^\top  \VV  \right)^{-1} \VV^\top  (V {A^3}^\top Z_0 + \bmeps)}_2^2 \\
&\le 2 \norm{\left(\VV^\top  \VV \right)^{-1}\VV^\top \bmeps}_2^2 + 2\norm{\left(\VV^\top  \VV  \right)^{-1} \VV^\top  V {A^3}^\top Z_0}_2^2 
\end{align*}
The first term can be bounded by $\frac{C_1}{\eta^2} \frac{r}{N_2}$ by lemma \ref{lem:bias-meta-test-I}. The second term can be bounded as follows.
\begin{align*}
&\norm{\left( \VV^\top  \VV  \right)^{-1} \VV^\top  V {A^3}^\top Z_0}_2^2  \\
&\le \norm{(\VV^\top \VV)^{-1}\VV^\top}_\op^2 \norm{V}_\op^2 \norm{{A^3}^\top Z_0}_2^2 
\\
&\le \norm{(\VV^\top \VV)^{-1}}_\op O(\norm{Y_0}_2^2  N_2 \lambda_{\max}(\Sigma)) r\norm{Z_0}_2^2 \\ &\left[\because \norm{(\VV^\top \VV)^{-1}\VV^\top}_\op^2 = \norm{\VV(\VV^\top \VV)^{-2}\VV^\top}_\op = \norm{(\VV^\top \VV)^{-1}}_\op \textrm{ and \cref{lem:bound-vv-norm}} \right] \\
&= O\left(\frac{ 1}{N_2 \eta^2 \norm{Y_0}_2^2 \lambda_{\min}(\Sigma) } \right) O(\norm{Y_0}_2^2  N_2 \lambda_{\max}(\Sigma)) r \norm{Z_0}_2^2 \\
&= O\left(C_2 \frac{r}{\eta^2} \right) \quad \textrm{for } C_2 = \frac{\lambda_{\max}(\Sigma)}{\lambda_{\min}(\Sigma)} \norm{Z_0}_2^2
\end{align*}
Therefore, we have bound $\norm{\widehat{{A^3}^\top Z_0}}_2^2$ by $\frac{C_1}{\eta^2} \frac{r}{N_2} + C_2\frac{r}{\eta^2}$. Substituting the upper bounds on $\norm{\widehat{{A^3}^\top Z_0}}_2^2$ and $\norm{\widehat{{A^3}^\top \hat{Z}_0} - {A^3}^\top Z_0}_2^2$ in equation \ref{eq:mse-upper-bound} establishes the desired bound.
\end{proof}

\begin{lemma}\label{lem:bias-meta-test-I}
Each covariate vector $X_i$ is mean-zero, satisfies $\E[X_iX_i^\top] = \Sigma$ and $\Sigma$-sub-gaussian. Additionally, suppose that$N_2 \ge O\left(r \left(\frac{\norm{Y_0}_2^2}{\eta^2} \frac{\lambda_{\max}(\Sigma)}{\lambda_{\min}(\Sigma)}\right)^2 \log(2/\delta_1)\right)$, and $\abs{Y_0^\top \hat{A^2}_i} \ge \eta \norm{Y_0}_2$ for all $i \in [r]$. Then with probability at least $1-\delta_1$ we have
$$
\norm{ \left( \VV^\top  \VV  \right)^{-1} \VV^\top \bmeps }_2^2 \le  \tilde{O}\left( \frac{r }{N_2 \eta^2 \norm{Y_0}_2^2 \lambda_{\min}(\Sigma)} \right)
$$
\end{lemma}
\begin{proof}
The bias term is given as $\norm{ (\VV^\top \VV)^{-1}\VV^\top \bmeps}_2^2 = \bmeps^\top \underbrace{\VV (\VV^\top \VV)^{-2}\VV^\top}_{:=M} \bm{\varepsilon}$. By the Hanson-Wright inequality (\cite{Vershynin18}, lemma 6.2.1) we have $$P\left(\abs{\bmeps^\top M \bmeps - E[\bmeps^\top M \bmeps]} \ge t\right) \le 2 \exp\left(-c\min\left(\frac{t^2}{\norm{M}^2_F}, \frac{t}{\norm{M}_\op}\right) \right).$$
Therefore, we have $\bmeps^\top M \bmeps \le E[\bmeps^\top M \bmeps] + O\left(\norm{M}_F \sqrt{\log(2/\delta_1)} \right) + O\left(\norm{M}_\op \log(2/\delta_1) \right)$ with probability at least $1-\delta_1/2$. From the singular value decomposition of $\VV$, it is easy to see that $\norm{M}_\op = \norm{\VV (\VV^\top \VV)^{-2}\VV^\top}_\op = \norm{(\VV^\top \VV)^{-1}}_\op$. Moreover, lemma \ref{lem:invertible} proves that with probability at least $1-\delta_1/2$, the matrix $M$ is invertible and $\norm{M}_\op \le  O\left( \frac{1 }{N_2 \eta^2 \norm{Y_0}_2^2 \lambda_{\min}(\Sigma)} \right)$ as long as $N_2 \ge O\left(r \left(\frac{\norm{Y_0}_2^2}{\eta^2} \frac{\lambda_{\max}(\Sigma)}{\lambda_{\min}(\Sigma)}\right)^2 \log(2/\delta_1)\right)$.

Since $\rank(M) \le \rank(\hGamma) \le r$, we have $\norm{M}_F \le \sqrt{r} \norm{M}_\op \le  O\left( \frac{\sqrt{r} }{N_2 \eta^2 \norm{Y_0}_2^2 \lambda_{\min}(\Sigma)} \right)$. By a similar argument we get $\E[\bmeps^\top M \bmeps] = \textrm{Tr}(M) \le r \norm{M}_{\op} \le  O\left( \frac{r }{N_2 \eta^2 \norm{Y_0}_2^2 \lambda_{\min}(\Sigma)} \right)$. This gives us $\bmeps^\top M \bmeps \le  \tilde{O}\left( \frac{r }{N_2 \eta^2 \norm{Y_0}_2^2 \lambda_{\min}(\Sigma)} \right)$.
\end{proof}

\begin{lemma}\label{lem:variance-meta-test-I}
Each covariate vector $X_i$ is mean-zero, satisfies $\E[X_iX_i^\top] = \Sigma$ and $\Sigma$-sub-gaussian. Additionally, assume that $\max\left\{\norm{\hat{A^1} - A^1}_F, \norm{\hat{A^2} - A^2}_F \right\} \le \delta$, and $\sin \theta(A^3, \hat{A^3}) \le \delta \sqrt{r}$. If $N_2 \ge O\left(r \left(\frac{\norm{Y_0}_2^2}{\eta^2} \frac{\lambda_{\max}(\Sigma)}{\lambda_{\min}(\Sigma)}\right)^2 \log(2/\delta_1)\right)$, and $\abs{Y_0^\top \hat{A^2}_i} \ge \eta \norm{Y_0}_2$ for all $i \in [r]$, then with probability at least $1-\delta_1$ we have
$$
\norm{ \left( \VV^\top  \VV  \right)^{-1} \VV^\top V{A^3}^\top Z_0  - {A^3}^\top Z_0}_2^2 = O\left( \frac{ \lambda_{\max}(\Sigma)}{\eta^2 \lambda_{\min}(\Sigma)} \norm{Z_0}_2^2 r \delta^2 \right)
$$
\end{lemma}
\begin{proof}
Our proof resembles the proof of Lemma 19 of \cite{TJJ20}, but there are some important differences. First note that, by lemma \ref{lem:esitmate-v-error} we can  write $V = \VV + E_V$ for a matrix $E_V$ with $\norm{E_V}_\op \le O(\norm{Y_0}_2 \sqrt{ N_2 \lambda_{\max}(\Sigma)} {\delta})$. This gives us the following bound on the variance.
\begin{align}
&\norm{ \left(\VV^\top  \VV \right)^{-1} \VV^\top V{A^3}^\top Z_0  - {A^3}^\top Z_0}_2^2 \nonumber \\
&= \norm{ \left( \VV^\top  \VV \right)^{-1} \VV^\top \VV {A^3}^\top Z_0  - {A^3}^\top Z_0 
+  \left( \VV^\top  \VV  \right)^{-1} \VV^\top E_V {A^3}^\top Z_0}_2^2 \nonumber \\
&= \norm{\left( \VV^\top  \VV  \right)^{-1} \VV^\top E_V {A^3}^\top Z_0}_2^2 \nonumber \\
&\le \norm{\left( \VV^\top  \VV  \right)^{-1} \VV^\top}_{\op}^2 \norm{E_V}_{\op}^2 \norm{{A^3}^\top Z_0}_2^2 \nonumber \\
&\le \norm{\VV \left( \VV^\top  \VV  \right)^{-2} \VV^\top}_{\op} O(\norm{Y_0}_2^2 { N_2 \lambda_{\max}(\Sigma)} {\delta^2}) r \norm{Z_0}_2^2 \label{eq:inter-term}
\end{align}
The last line uses $\norm{{A^3}^\top Z_0}_2^2 = \sum_{i=1}^r ({A^3}^\top_i Z_0)^2 \le \sum_{i=1}^r \norm{A^3}_2^2 \norm{Z_0}_2^2 = r \norm{Z_0}_2^2$. Now $\norm{\VV \left( \VV^\top  \VV  \right)^{-2} \VV^\top}_{\op} = \norm{ \left( \VV^\top  \VV  \right)^{-1} }_{\op}$ and lemma \ref{lem:invertible} proves that with probability at least $1-\delta_1/2$, the matrix $\VV^\top \VV$ is invertible and $\norm{\left(\VV^\top \VV\right)^{-1}}_\op \le  O\left( \frac{1 }{N_2 \eta^2 \norm{Y_0}_2^2 \lambda_{\min}(\Sigma)} \right)$ as long as $N_2 \ge O\left(r \left(\frac{\norm{Y_0}_2^2}{\eta^2} \frac{\lambda_{\max}(\Sigma)}{\lambda_{\min}(\Sigma)}\right)^2 \log(2/\delta_1)\right)$. Substituting the upper bound on the operator norm of $\left(\VV^\top \VV \right)^{-1}$ gives the desired bound.

\end{proof}

\begin{lemma}\label{lem:bound-vv-norm}
 If $N_2 \ge O(r \log(1/\delta_1))$ then we have
$$\norm{\VV}_\op \le O\left(\norm{Y_0}_2 \sqrt{ N_2 \lambda_{\max}(\Sigma)}\right)$$
with probability at least $1-\delta_1$.
\end{lemma}
\begin{proof}
\begin{align*}
\norm{\VV}_\op^2 &= 
  \lambda_{\max}\left((Y_0^\top \hat{A^2} \odot \calX \hat{A^1})^\top (Y_0^\top \hat{A^2} \odot \calX \hat{A^1}) \right) \\
&=  N_2 \lambda_{\max}\left(U^\top (\frac{1}{N_2}\calX^\top \calX) U \right)
\end{align*}
where in the last line we write $U  \in \R^{d_1 \times r}$ to denote the matrix with columns $U_i = (Y_0^\top \hat{A^2_i}) \hat{A^1_i}$. The matrix $U$ has orthogonal columns and $\norm{U}_\op \le \norm{Y_0}_2$. Therefore, we can apply lemma \ref{lem:concentration-proj-covariance} to obtain that as long as $N_2 \ge r\log(1/\delta_1)$, we have $\lambda_{\max}\left(U^\top (\frac{1}{N_2}\calX^\top \calX) U \right)$ is bounded by $O\left(\norm{U^\top \Sigma U}_\op + \lambda_{\max}(\Sigma) \norm{Y_0}_2^2 \right)$ with probability at least $1-\delta_1$. Moreover, $\norm{U^\top \Sigma U}_\op$ is bounded by $\lambda_{\max}(\Sigma) \norm{Y_0}_2^2$. This establishes a bound of $O(N_2 \lambda_{\max}(\Sigma) \norm{Y_0}_2^2)$ on $\norm{\VV}_\op^2$.
\end{proof}

\begin{lemma}\label{lem:esitmate-v-error}
Suppose, $\max\left\{\norm{\hat{A^1} - A^1}_F, \norm{\hat{A^2} - A^2}_F \right\} \le \delta$. If $N_2 \ge O(r \log(1/\delta_1))$ then we have
$$\norm{\VV - V}_\op \le O\left(\norm{Y_0}_2 \sqrt{ N_2 \lambda_{\max}(\Sigma)}{\delta}\right)$$
with probability at least $1-\delta_1$.
\end{lemma}
\begin{proof}
We will write $\hat{A^1} = A^1 + E^1$, and $\hat{A^2} = A^2 + E^2$. Note that we have $\norm{E^1}_F \le \delta$, $\norm{E^2}_F \le \delta $.
\begin{align*}
\norm{\VV - V}_{\op} &= \norm{(Y_0^\top \hat{A^2} \odot \calX \hat{A^1}) - (Y_0^\top {A^2} \odot \calX {A^1})}_\op \\
&= \norm{(Y_0^\top ({A^2} + E^2) \odot \calX ({A^1} + E^1) )  - (Y_0^\top {A^2} \odot \calX {A^1})}_\op\\
&\le \norm{(Y_0^\top E^2) \odot (\calX \hat{A^1})}_\op + \norm{(Y_0^\top A^2) \odot (\calX E^1)}_\op + \norm{(Y_0^\top E^2) \odot (\calX E^1) }_\op \\
\end{align*}
Consider the first term. 
\begin{align*}
 &\norm{(Y_0^\top E^2) \odot (\calX \hat{A^1})}_\op^2 
 \\&=  \lambda_{\max}\left(\left((Y_0^\top E^2) \odot (\calX \hat{A^1})\right)^\top (Y_0^\top E^2) \odot (\calX \hat{A^1}) \right) \\
 &= N_2 \lambda_{\max}\left(U^\top (\frac{1}{N_2}\calX^\top \calX) U\right)\\
\end{align*}
where in the last line we write $U \in \R^{d_1 \times r}$ to denote the matrix with columns $U_i = (Y_0^\top E^2_i)  \hat{A^1}_i$. Note that $U$ has orthogonal columns as the columns of $\hat{A^1}$ are orthogonal. Moreover, $\norm{U}_\op \le \norm{U}_F \le \norm{Y_0}_2 \norm{E^2}_F \le \norm{Y_0}_2 {\delta}$. Therefore, we can apply lemma \ref{lem:concentration-proj-covariance} to get that as long as $N_2 \ge r \log(1/\delta_1)$, we have $\lambda_{\max}(U^\top (\frac{1}{N_2}\calX^\top \calX) U)$ is bounded by $O\left(\norm{U^\top \Sigma U}_\op + \lambda_{\max}(\Sigma)\delta^2 \norm{Y_0}_2^2\right)$ with probability at least $1-\delta_1$. Moreover, $\norm{U^\top \Sigma U}_\op = \lambda_{\max}(\Sigma) \norm{U}_F^2 \le \lambda_{\max}(\Sigma)\delta^2\norm{Y_0}_2^2$. This establishes a bound of $O(\sqrt{\lambda_{\max}(\Sigma) N_2}\delta \norm{Y_0}_2)$ on $\norm{(Y_0^\top E^2) \odot (\calX \hat{A^1})}_\op$. By a similar argument, one can establish a bound of $O(\sqrt{\lambda_{\max}(\Sigma) N_2}\delta\norm{Y_0}_2)$ on the second term $\norm{(Y_0^\top A^2) \odot (\calX E^1)}_\op$, and a bound of $O(\sqrt{\lambda_{\max}(\Sigma) N_2}\delta\norm{Y_0}_2)$ on the third term $\norm{(Y_0^\top A^2) \odot (\calX A^1) E^0}_\op$.
\end{proof}

\begin{lemma}\label{lem:concentration-proj-covariance}
Suppose each covariate $x_i$ is mean-zero, satisfies $\E[xx^\top] = \Sigma$ and $\Sigma$-subgaussian. Moreover, $A$ and $B$ are rank $r$  matrices with orthogonal columns. Then the following holds
$$
\norm{A^\top \frac{\calX^\top \calX}{n}B - A^\top \Sigma B}_\op \le O\left(\lambda_{\max}(\Sigma)\max\{\norm{A}_\op^2, \norm{B}_\op^2\} \left( \sqrt{\frac{r}{n}} + \frac{r}{n} + \sqrt{\frac{\log(1/\delta)}{n}} + \frac{\log(1/\delta)}{n}\right)\right)
$$
with probability at least $1-\delta$.
\end{lemma}
\begin{proof}
The proof is very similar to the proof of lemma 20 from \cite{TJJ20}.
\end{proof}
  
\begin{lemma}\label{lem:invertible}
Suppose, $N_2 \ge O\left(r \left(\frac{\norm{Y_0}_2^2}{\eta^2 } \frac{\lambda_{\max}(\Sigma)}{\lambda_{\min}(\Sigma)}\right)^2 \log(1/\delta_2)\right)$, and $\abs{Y_0^\top \hat{A^2}_i} \ge \eta \norm{Y_0}_2$ for all $i \in [r]$. Then the matrix $( \VV^\top \VV )$ is invertible and 
$$
\norm{(\VV^\top \VV)^{-1}}_\op \le O\left( \frac{1 }{N_2 \eta^2 \norm{Y_0}_2^2 \lambda_{\min}(\Sigma)} \right)
$$
with probability at least $1-\delta_2$.
\end{lemma}
\begin{proof}
From the definition of the matrix $\VV$, we have $$\VV^\top \VV  =  (Y_0^\top \hat{A^2} \odot \calX \hat{A^1})^\top (Y_0^\top \hat{A^2} \odot \calX \hat{A^1})  .$$ If we define a matrix $U \in \R^{d_1 \times r}$ with columns $U_i =   (Y_0^\top \hat{A^2_i}) \hat{A^1_i}$, then it can be verified that $\frac{1}{N_2}\VV^\top \VV  =  U^\top \left(\frac{1}{N_2} \calX^\top \calX \right) U $. This gives us $\E[\frac{1}{N_2}\VV^\top \VV ] = U^\top \Sigma U $. Therefore, we can write $\frac{1}{N_2} \VV^\top \VV=  \calE  +   U^\top \Sigma U $, for a matrix $\calE$ with $\E[\calE] = 0$. Since matrix $U$ has orthogonal columns and $\norm{U}_\op \le  \norm{Y_0}_2 $ we can apply lemma \ref{lem:concentration-proj-covariance} to conclude that as long as $N_2 \ge O(r \log(1/\delta_1))$ we have $\norm{\calE}_\op \le O\left(\lambda_{\max}(\Sigma)  \norm{Y_0}_2^2 \sqrt{r/N_2} \right)$. On the other hand,
\begin{align*}
\lambda_{\min}(U^\top \Sigma U ) = \min_{x \in \R^{d_3}, x \neq 0} \frac{x^\top U^\top \Sigma U  x}{x^\top x} \ge  \eta^2 \norm{Y_0}_2^2 \min_{w \in \R^{d_1}, w \neq 0} \frac{w^\top \Sigma w}{w^\top w} = \eta^2 \norm{Y_0}_2^2 \lambda_{\min}(\Sigma).
\end{align*}
The first inequality follows from substituting $w = U  x$ and observing that $w^\top w = x^\top  U^\top U  x \ge \min_i  \abs{Y_0^\top \hat{A^2}_i}^2 x^\top  x \ge   \norm{Y_0}_2^2 \eta^2 x^\top x$.
Therefore, 
\begin{align*}
\lambda_{\min}\left(\frac{1}{N_2}  \VV^\top \VV \right) &\ge \lambda_{\min}(U^\top \Sigma U ) - \lambda_{\max}(\calE)
\\&\ge O\left( \norm{Y_0}_2^2 \eta^2 \lambda_{\min}(\Sigma) \right) - \norm{\calE}_\op 
\\& \ge  O\left( \norm{Y_0}_2^2 ( \eta^2 \lambda_{\min}(\Sigma) - \lambda_{\max}(\Sigma)  \sqrt{r/N_2} )\right)
\end{align*} 
Therefore, as long as $N_2 \ge r \left(\frac{1}{\eta^2} \frac{\lambda_{\max}(\Sigma)}{\lambda_{\min}(\Sigma)}\right)^2$, $\frac{1}{N_2} \VV^\top \VV $ is invertible and so is $\VV^\top \VV$.

Now, $(\VV^\top \VV )^{-1} = \frac{1}{N_2} (\frac{1}{N_2} \VV^\top \VV )^{-1} = \frac{1}{N_2}(\calE  + U^\top \Sigma U )^{-1}$. Moreover, 
\begin{align*}
\norm{(U^\top \Sigma U)^{-1} \calE}_\op &\le \norm{(U^\top \Sigma U)^{-1}}_\op\norm{\calE}_\op =  \frac{\norm{\calE}_\op}{\lambda_{\min}(U^\top \Sigma U)} = O\left(\frac{1}{\eta^2} \frac{\lambda_{\max}(\Sigma)}{\lambda_{\min}(\Sigma)} \sqrt{\frac{r}{N_2}} \right)\\
\end{align*}
Therefore, as long as $N_2 \ge O\left(\frac{r}{\eta^2} \frac{\lambda_{\max}^2(\Sigma)}{\lambda_{\min}^2(\Sigma)}  \right)$ we have, $\norm{(U^\top \Sigma U)^{-1} \calE}_\op \le 1/4$. Therefore, we can apply lemma \ref{lem:inverse-formula} to conclude that $(\frac{1}{N_2} \VV^\top \VV )^{-1} = ( U^\top \Sigma U )^{-1} + F$ where $\norm{F}_\op \le \frac{1}{3}\norm{(U^\top \Sigma U)^{-1}}_\op$. Therefore, $\norm{(\VV^\top \VV)^{-1}}_\op \le \frac{4}{3N_2} \norm{(U^\top \Sigma U)^{-1}}_\op = \frac{4}{3N_2} \frac{1}{\lambda_{\min}(U^\top \Sigma U)} \le \frac{4 }{3N_2 \eta^2 \norm{Y_0}_2^2 \lambda_{\min}(\Sigma)}$.
\end{proof}

\begin{lemma}[Restated lemma 23 from \cite{TJJ20}]\label{lem:inverse-formula}
Let $A$ be a positive-definite matrix and $E$ is another matrix satisfying $\norm{EA^{-1}} \le \frac{1}{4}$. Then $(A + E)^{-1} = A^{-1} + F$ where $\norm{F}_\op \le \frac{4}{3} \norm{A^{-1}}_\op \norm{EA^{-1}}_{\op}$.
\end{lemma}
\section{Proof of Theorem \ref{thm:method_of_moments}}

We first recall the method of moments estimator from \cite{TJJ20}. If the response $R_i = X_i^\top B \alpha_{t(i)}$ and each $X_i \simiid \Normal(0, I_{d_1})$ then we have $\E\left[\frac{1}{N}\sum_{i=1}^N R_i^2 X_i X_i^\top \right] = 2 \bar{\Gamma} + (1+\trace(\bar{\Gamma})) I_{d_1}$ where $\bar{\Gamma} = \frac{1}{N} \sum_{i=1}^N B \alpha_{t(i)} \alpha_{t(i)}^\top B^\top$. 
 If we write $\bar{\Lambda} = \frac{1}{N} \sum_{i=1}^N \alpha_{t(i)} \alpha_{t(i)}^\top$ to be the empirical task matrix we have $\E\left[\frac{1}{N}\sum_{i=1}^N R_i^2 X_i X_i^\top \right]  = B (2\bar{\Lambda})B^\top + B_\perp (1 + \trace(\bar{\Gamma}))\Identity_r B^\top_\perp$. So that we can recover $B$ from the top $r$ singular values of the statistic $\frac{1}{N}\sum_{i=1}^N R_i^2 X_i X_i^\top $. Moreover, theorem 3 of \cite{TJJ20} proves that such an estimate $\hat{B}$ satisfies $\sin \theta(\hat{B}, B) \le \sqrt{\frac{\kappa}{\nu} \frac{d_1 r}{N}}$, where $\nu = \sigma_r(\bar{\Lambda})$ and $\kappa = \trace(\bar{\Lambda})/(r \nu)$.
 
 
\paragraph{Recovering $A^1$}. Let us consider the estimation of the first factor $A^1$. The response of the $i$-th individual is given as
 \[
 R_i = X_i^\top A_{(1)} (Z_{t(i)} \otimes Y_{t(i)}) + \veps_i = X_i^\top A^{1} \underbrace{(A^{3} \odot A^{2})^\top (Z_{t(i)} \otimes Y_{t(i)})}_{:= P_{t(i)}^1} + \veps_i
 \]
Therefore, we  recover $A^1$ from the top $r$ singular values of $\frac{1}{N} \sum_{i=1}^N R_i^2 X_i X_i^\top$. 
In order to obtain a bound on $\sin \theta(A^1, \hat{A^1})$ we need to bound eigenvalue and trace of the empirical task matrix $\bar{\Lambda} = \frac{1}{N} \sum_{i=1}^N P^1_{t(i)} {P^1_{t(i)}}^\top$. Since each $t(i)$ is a uniform random draw from $\set{1,\ldots,T}$, we have $\Lambda = \E[\Bar{\Lambda}] = \frac{1}{T} \sum_{t=1}^T P^1_t {P^1_t}^\top = (A^3 \odot A^2)^\top \frac{1}{T}\sum_{t=1}^T (Z_t \otimes Y_t)(Z_t \otimes Y_t)^\top (A^3 \odot A^2) = \frac{1}{T}  (A^3 \odot A^2)^\top (\calZ^\top \odot \calY^\top)(\calZ^\top \odot \calY^\top)^\top (A^3 \odot A^2)$. We first bound the eigenvalues of $\Lambda$ and then use matrix concentration inequality to bound the eigenvalues of the empirical task matrix $\bar{\Lambda}$.
$$
\lambda_{\min}(\Lambda) = \frac{1}{T} \lambda_{\min}((\calZ^\top \odot \calY^\top)(\calZ^\top \odot \calY^\top)^\top) = \frac{1}{T} \lambda_{\min}\left( \calZ^\top \calZ \otimes  \calY^\top \calY \right) = \frac{1}{T} \lambda_{\min}(\calZ^\top \calZ ) \lambda_{\min}(\calY^\top \calY) 
$$
The first equality follows from the observation that $A^2 \odot A^1$ has orthonormal columns, and the last equality follows because the eigenvalues of Kronecker product of two matrices are given as the Kronecker product of eigenvalues of the two matrices. Since each $Z_t \simiid \Normal(0,\Identity_{d_3})$, the minimum singular value of $\calZ$ is bounded by $\sqrt{T} - \sqrt{d_3}$ with probability at least $1 - 2 \exp(-O(d_3))$ (see e.g. theorem 4.6.1 of \cite{Vershynin18}). This implies that $\lambda_{\min}(\calZ^\top \calZ) = \sigma_{\min}(\calZ)^2 \ge T/4$ with probability at least $1 - 2 \exp(-O(d_3))$ as long as $T \ge 4d_3$. Similarly, it can be shown that $\lambda_{\min}(\calY^\top \calY) \ge T/4$ with probability at least $1 - 2 \exp(-O(d_2))$ as long as $T \ge 4d_2$. This establishes a high probability lower bound of $T/16$ on $\lambda_{\min}(\Lambda)$.

Moreover, for any $i \in [N]$, 
\begin{align*}
\lambda_{\max}(P^1_{t(i)}{P^1_{t(i)}}^\top) &= \lambda_{\max}((Z_{t(i)} \otimes Y_{t(i)})(Z_{t(i)} \otimes Y_{t(i)})^\top) = \lambda_{\max}(Z_{t(i)} Z_{t(i)}^\top \otimes Y_{t(i)}Y_{t(i)}^\top)\\
&\le \lambda_{\max}(Z_{t(i)} Z_{t(i)}^\top) \lambda_{\max}( Y_{t(i)}Y_{t(i)}^\top) \le \norm{Z_{t(i)}}_2^2 \norm{Y_{t(i)}}_2^2 
\end{align*}
When $Z_{t(i)}$ is drawn from standard Normal distribution $\norm{Z_{t(i)}}_2 \le 2\sqrt{d_3}$ with probability at least $1 - \exp(-O(d_3))$. By a union bound over all $T$ tasks we have for all $t \in [T]$, $\norm{Z_{t}}_2 \le 2\sqrt{d_3}$ with probability at least $1 - T \exp(-O(d_3))$. A similar argument shows that for all $t \in [T]$, $\norm{Y_{t}}_2 \le 2\sqrt{d_2}$ with probability at least $1 - T \exp(-O(d_2))$. Therefore, we are guaranteed that $\lambda_{\max}(P^1_{t(i)} {P^1_{t(i)}}^\top) \le 16d_2d_3$ for all $i$, with probability at least $1 - T \exp(-O(\min\{d_2,d_3\}))$. Now we can apply matrix concentration inequality (lemma~\ref{lem:Matrix-Chernoff}) to derive the following result.
\begin{align*}
    P\left(\lambda_{\min}(\Bar{\Lambda}) \le \frac{T}{32} \right) \le d_1 \exp\left\{-O\left(\frac{TN}{d_2 d_3}\right)\right\}.
\end{align*}
Similarly, we can establish an upper bound on the maximum eigenvalue of $\Bar{\Lambda}$.
\begin{align*}
    P\left(\lambda_{\max}(\Bar{\Lambda}) \ge 32 T \right) \le d_1 \exp\left\{-O\left(\frac{TN}{d_2 d_3}\right)\right\}.
\end{align*}
Therefore, $\trace(\Bar{\Lambda}) \le r \lambda_{\max}(\Bar{\Lambda}) \le 32 r T$ with probability at least $1 - d_1 \exp\left\{-O\left({TN}/{d_2 d_3}\right)\right\}$. Moreover, $\kappa = \trace(\Bar{\Lambda})/(r \lambda_{\min}(\Bar{\Lambda}) = O(1)$. This implies the following bound on the distance between $\hat{A^1}$ and $A^1$.
\begin{align*}
    \sin \theta\left(\hat{A^1}, A^1 \right) \le O\left(\sqrt{\frac{\kappa d_1 r}{\nu N}} \right) = O\left(\sqrt{\frac{ d_1 r}{TN}}\right).
\end{align*}

\paragraph{Recovering $A^2$}. We can provide a bound on the error in estimating $A^2$ through a similar approach. The response of individual $i$ can be written as
\[
 R_i = Y_{t(i)}^\top A_{(2)} (Z_{t(i)} \otimes X_i) + \veps_i = Y_{t(i)}^\top A^{2} \underbrace{W (A^{3} \odot A^{1})^\top (Z_{t(i)} \otimes X_i)}_{:= P_{t(i)}^2} + \veps_i
 \]
Therefore, we can recover $A^2$ from the top $r$ singular values of $\frac{1}{N} \sum_{i=1}^N R_i^2 Y_{t(i)} Y_{t(i)}^\top$. 
Now the empirical task matrix is $\bar{\Lambda} = \frac{1}{N} \sum_{i=1}^N P^2_{t(i)} {P^2_{t(i)}}^\top$. We now bound the eigenvalue and trace of the empirical task matrix. Since each $t(i)$ is a uniform random draw from $\set{1,\ldots,T}$, we have $\Lambda = \E[\Bar{\Lambda}] = \frac{1}{T} \sum_{t=1}^T P^2_t {P^2_t}^\top = (A^3 \odot A^1)^\top \frac{1}{T} \sum_{t=1}^T \E[(Z_t \otimes X)(Z_t \otimes X)^\top ] (A^3 \odot A^1) = (A^3 \odot A^1)^\top \frac{1}{T} \sum_{t=1}^T (Z_t \otimes \Identity_{d_1}) (Z_t \otimes \Identity_{d_1})^\top (A^3 \odot A^1) = \frac{1}{T} (A^3 \odot A^1)^\top (\calZ^\top  \otimes \Identity_{d_1}) (\calZ^\top  \otimes \Identity_{d_1})^\top (A^3 \odot A^1)$
\begin{align*}
    \lambda_{\min}(\Lambda) = \frac{1}{T} \sigma_{\min}((\calZ^\top  \otimes \Identity_{d_1}) (\calZ^\top  \otimes \Identity_{d_1})^\top) = \frac{1}{T} \sigma_{\min} (\calZ^\top \calZ \otimes \Identity_{d_1}) \ge \frac{1}{T} \sigma_{\min}(\calZ^\top \calZ)
\end{align*}
Since each $Z_t \simiid \Normal(0,\Identity_{d_3})$, the minimum singular value of $\calZ$ is bounded by $\sqrt{T} - \sqrt{d_3}$ with probability at least $1 - 2 \exp(-O(d_3))$ (see e.g. theorem 4.6.1 of \cite{Vershynin18}). This implies that $\lambda_{\min}(\calZ^\top \calZ) = \sigma_{\min}(\calZ)^2 \ge 1/4$ with probability at least $1 - 2 \exp(-O(d_3))$ as long as $T \ge 4d_3$. Moreover, for any $i \in [N]$, 
\begin{align*}
\lambda_{\max}(P^2_{t(i)}{P^2_{t(i)}}^\top ) &= \lambda_{\max}((Z_{t(i)} \otimes X_i)(Z_{t(i)} \otimes X_i)^\top) = \lambda_{\max}(Z_{t(i)} Z_{t(i)}^\top \otimes X_i X_i^\top)\\
&\le \lambda_{\max}(Z_{t(i)} Z_{t(i)}^\top) \lambda_{\max}( X_i X_i^\top) \le \norm{Z_{t(i)}}_2^2 \norm{X_i}_2^2 
\end{align*}
When $Z_{t(i)}$ is drawn from standard Normal distribution $\norm{Z_{t(i)}}_2 \le 2\sqrt{d_3}$ with probability at least $1 - \exp(-O(d_3))$. By a union bound over all $T$ tasks we have for all $t \in [T]$, $\norm{Z_{t}}_2 \le 2\sqrt{d_3}$ with probability at least $1 - T \exp(-O(d_3))$. A similar argument shows that for all $i \in [N]$, $\norm{X_i}_2 \le 2\sqrt{d_1}$ with probability at least $1 - T \exp(-O(d_1))$. Therefore, we are guaranteed that $\lambda_{\max}(P^1_{t(i)} {P^1_{t(i)}}^\top) \le 16d_1 d_3$ for all $i$, with probability at least $1 - N \exp(-O(\min\{d_1,d_3\}))$. Now we can apply matrix concentration inequality (lemma~\ref{lem:Matrix-Chernoff}) to derive the following result.
\begin{align*}
    P\left(\lambda_{\min}(\Bar{\Lambda}) \le \frac{1}{8} \right) \le d_2 \exp\left\{-O\left(\frac{N}{d_1 d_3}\right)\right\}.
\end{align*}
Similarly, we can establish an upper bound on the maximum eigenvalue of $\Bar{\Lambda}$.
\begin{align*}
    P\left(\lambda_{\max}(\Bar{\Lambda}) \ge 8 \right) \le d_2 \exp\left\{-O\left(\frac{N}{d_1 d_3}\right)\right\}.
\end{align*}
Therefore, $\trace(\Bar{\Lambda}) \le 8 r$ with probability at least $1 - d_1 \exp\left\{-O\left({N}/{d_1 d_3}\right)\right\}$. Moreover, $\kappa = \trace(\Bar{\Lambda})/(r \lambda_{\min}(\Bar{\Lambda})) = O(1)$. This implies the following bound on the distance between $\hat{A^2}$ and $A^2$.
\begin{align*}
    \sin \theta\left(\hat{A^2}, A^2 \right) \le O\left(\sqrt{\frac{\kappa d_2 r}{\nu N}} \right) = O\left(\sqrt{\frac{ d_2 r}{N}}\right).
\end{align*}

\begin{lemma}[Restated theorem 5.1.1 from \cite{Tropp15}]\label{lem:Matrix-Chernoff}
Consider a sequence of $\{X_k\}_{k=1}^N$ independent, random, Hermitian matrices of dimension $d\times d$. Assume that the eigenvalues of each $X_k$ is bounded between $[0,L]$. Let $Y = 1/N \sum_{k} X_k$, $\mu_{\min} = \lambda_{\min}(\E[Y])$, and $\mu_{\max} = \lambda_{\max}(\E[Y])$. Then we have
\begin{align*}
    P\left(\lambda_{\min}(Y) \le (1 - \veps)\mu_{\min} \right) &\le d \left[\frac{e^{-\veps}}{(1-\veps)^{1-\veps}} \right]^{\mu_{\min}N/L} \ \forall \veps \in [0,1)\\
        P\left(\lambda_{\max}(Y) \ge (1 + \veps)\mu_{\max} \right) &\le d \left[\frac{e^{\veps}}{(1+\veps)^{1+\veps}} \right]^{\mu_{\max}N/L} \ \forall \veps \in [0,\infty)\\
\end{align*}
\end{lemma}
\section{Meta-Test for Method-of-Moments Based Estimation}

In the meta-test phase, $(X_{i}, R_{i})$ for $i = 1, \ldots, N_{2}$ are observed for a task with specific feature $Y_{0}$. The model can be expressed as 
\begin{equation*}
    R_{i} = (Y_{0} \otimes X_{i})^\top (A^{2} \odot A^{1})A^{3\top} Z_{0} + \varepsilon_{i},  \quad \textrm{for } i = 1, \ldots, N_{2}.
\end{equation*}
If we denote the latent task factor $A^{3\top} Z_{0}$ as a vector $\alpha \in \R^{r}$, $\alphav$ can be estimated from the least square problem with $A^1$ and $A^2$ substituted by their estimators from the meta-training phase
\begin{equation*}
    \hat{\alphav} = \textrm{argmin}_{\alphav} \norm{\mathbf{R} - (Y_{0} \otimes \calX^\top)^{\top}(\hat{A}^{2} \odot \hat{A}^{1})\alphav}_{2}.
\end{equation*}
For notation simplicity, throughout this section we denote $(Y_{0} \otimes \calX^\top)^\top (\hat{A}^{2} \odot \hat{A}^1 )$ as $\MM$, and $(Y_0 \otimes \calX^\top)^\top (A^2 \odot A^1)$ as $M$. In addition, we let $\Mo$ denote $Y_{0}^\top \hat{A}^{2} \odot X_{0}^\top \hat{A}^{1}$, $M_{0}$ denote $Y_{0}^\top A^{2} \odot X_{0}^\top A^{1}$. Then the least square estimation becomes
\begin{equation*}
    \hat{\alphav} = [\MM^\top \MM]^{-1} \MM^\top \mathbf{R}.
\end{equation*}
After obtaining the estimation of the task with observable and latent features $Y_{0}$ and $\alphav$, a test sample is collected on this task with input $X_{0}$. Then the estimation error can be expressed using the notation as
\begin{align*}
    & \E_{X_{0}}\left[\left(A(X_{0}, Y_{0}, \alphav) - \hat{A}(X_{0}, Y_{0}, \hat{\alpha})\right)^{2}\right] \\
    & = \E_{X_{0}}\left[\left((Y_{0}\otimes X_{0})^\top (A^2 \odot A^1)\alphav - (Y_{0} \otimes X_{0})^\top(\hat{A}^{2} \odot \hat{A}^{1})\hat{\alphav}\right)^2 \right] \\
    & = \E_{X_{0}}\left[ \norm{\Mo \hat{\alphav} - M_{0}\alphav }_{2}^{2} \right].
\end{align*}
Formally, we have
\begin{theorem}
Suppose each covariate $x_i$ is mean-zero, satisfies $\E[xx^\top] = \Sigma$ and $\Sigma$-subgaussian, and $\varepsilon_{i}$'s are i.i.d. mean -zero, sub-gaussian variables with variance parameter 1, independent of $x_i$. If $\abs{Y_0^\top \hat{A^2}_i} \ge \eta \norm{Y_0}_2$ for all $i \in [r]$, and $N_{2} \geq O\left((r+\log 2/\delta_{2}) \left(\frac{1}{\eta^2}\frac{\lambda_{\max}(\Sigma)}{\lambda_{\min}(\Sigma)}\right)^2 \right)$, then with probability at least $1 - \delta_{2}$, we have
\begin{equation*}
    \E_{X_{0}}\left[\left(A(X_{0}, Y_{0}, \alphav) - \hat{A}(X_{0}, Y_{0}, \hat{\alpha})\right)^{2}\right] = O \left( C_{1} r \delta^2 + C_{2} \frac{r}{N_{2}} \right),
\end{equation*}
for $C_{1} = \E\left[\norm{X_{0}}_{2}^{2}\right] \norm{Y_{0}}_{2}^{2} \left( \frac{\lambda_{\max}(\Sigma)}{\lambda_{\min}(\Sigma)} \right)^2 \frac{1}{\eta^4} \norm{\alphav}_{2}^{2}$ and $C_{2} = \E\left[\norm{X_{0}}_{2}^{2}\right] \frac{\lambda_{\max}(\Sigma)}{\lambda_{\min}^{2}(\Sigma)}\frac{1}{\eta^{4}}$.
\end{theorem}
\begin{proof}
The error can be written as
\begin{align*}
    & \Mo \hat{\alphav} - M_{0} \alphav \\
    = & \Mo (\MM^\top \MM)^{-1} \MM^\top (M\alphav + \calE) - M_{0} \alphav \\
    = & \Mo (\MM^\top \MM)^{-1} \MM^\top (\MM + M - \MM) \alphav + \Mo(\MM^\top \MM)^{-1} \MM^\top \calE - M_{0} \alphav \\
    = & (\Mo - M_{0}) \alphav - \Mo (\MM^\top \MM)^{-1} \MM^\top (\MM - M)\alphav + \Mo (\MM^\top \MM)^{-1} \MM \calE.
\end{align*}
Thus, by Lemma \ref{lem:variance-error}, \ref{lem:Mhat-M-inverse}, \ref{lem:MM-M}, \ref{lem:Mhat-Mhat}, \ref{lem:Mo-M0}, \ref{lem:Mo}, we have
\begin{align*}
    & \norm{\Mo \hat{\alphav} - M_{0} \alphav}_{2} \\
    & \leq \norm{\Mo - M_{0}}_{\op} \norm{\alphav}_{2} + \norm{\Mo}_{\op} \norm{(\MM^\top \MM)^{-1}}_{\op} \norm{\MM}_{\op} \norm{\MM - M}_{\op}\norm{\alphav}_{2} \\
    & + \norm{\Mo (\MM^\top \MM)^{-1} \MM^\top \calE}_{2} \\
    & \leq O \left(\norm{\alphav}_{2} \sqrt{r} \delta \norm{X_{0}}_{2} \norm{Y_{0}}_{2} \right) + O \left( \norm{X_{0}}_{2} \norm{Y_{0}}_{2} \frac{1}{\eta^2 \norm{Y_{0}}_{2}^{2} \lambda_{\min}(\Sigma)} \frac{1}{N_{2}} \right. \\
    & \cdot \left. \sqrt{N_{2} \lambda_{\max}(\Sigma) }\norm{Y_{0}}_{2} \cdot \norm{Y_{0}}_{2} \delta \sqrt{N_{2} r \lambda_{max}(\Sigma)} \norm{\alphav}_{2} \right) \\
    & + O \left(\frac{\norm{X_{0}}_{2}}{\sqrt{\lambda_{\max}(\Sigma)}} \frac{1}{\sqrt{N_{2}}} \sqrt{r+\log 2/\delta_{2}}\frac{1}{\eta^2}\frac{\lambda_{\max}(\Sigma)}{\lambda_{\min}(\Sigma)} \right) \\
    & = O\left( \norm{X_{0}}_{2} \norm{Y_{0}}_{2} \delta \frac{\lambda_{\max}(\Sigma)}{\lambda_{\min}(\Sigma)} \frac{1}{\eta^2} \sqrt{r} \norm{\alphav}_{2}\right) + O \left( \norm{X_{0}}_{2} \frac{\sqrt{\lambda_{\max}(\Sigma)}}{\lambda_{\min}(\Sigma)} \frac{1}{\eta^2} \sqrt{r + \log 2/\delta_{2}} \frac{1}{\sqrt{N_{2}}} \right).
\end{align*}
\end{proof}

\begin{lemma}\label{lem:variance-error}
Suppose each covariate $x_i$ is mean-zero, satisfies $\E[xx^\top] = \Sigma$ and $\Sigma$-subgaussian, and $\varepsilon_{i}$'s are i.i.d. mean -zero, sub-gaussian variables with variance parameter 1, independent of $x_i$. If $\abs{Y_0^\top \hat{A^2}_i} \ge \eta \norm{Y_0}_2$ for all $i \in [r]$, and $N_{2} \geq O\left((r+\log 2/\delta_{2}) \left(\frac{1}{\eta^2}\frac{\lambda_{\max}(\Sigma)}{\lambda_{\min}(\Sigma)}\right)^2 \right)$, we have
\begin{equation*}
    \norm{\Mo (\MM^\top \MM)^{-1} \MM^\top \calE}_{2}^{2} \leq O \left(\frac{\norm{X_{0}}_{2}^{2}}{\lambda_{\max}(\Sigma)} \frac{1}{N_{2}} (r+\log 2/\delta_{2})\left(\frac{1}{\eta^2}\frac{\lambda_{\max}(\Sigma)}{\lambda_{\min}(\Sigma)}\right)^2 \right),
\end{equation*}
with probability at least $1-\delta_{2}/2$.
\end{lemma}
\begin{proof}
Note that 
\begin{equation*}
    \norm{\Mo (\MM^\top \MM)^{-1}\MM^\top \calE}_{2}^{2} = \calE^\top G \calE,
\end{equation*}
with $G = \MM (\MM^\top \MM)^{-1} \Mo^{\top} \Mo (\MM^\top \MM)^{-1} \MM^\top$. By the Hanson-Wright inequality (\cite{Vershynin18}, lemma 6.2.1) we have $$P\left(\abs{\bmeps^\top G \bmeps - E[\bmeps^\top G \bmeps]} \ge t\right) \le 2 \exp\left(-c\min\left(\frac{t^2}{\norm{G}^2_F}, \frac{t}{\norm{G}_\op}\right) \right).$$ Thus, $\calE^\top G \calE \leq \E[\calE^\top G \calE] + O\left( \norm{G}_{F} \sqrt{\log (2/\delta_{1})} \right) + O\left( \norm{G}_{\op} \log (2/\delta_{1})\right)$, with probability at least $1-\delta_{1}/2$. By Lemma \ref{lem:Mhat-M-inverse} \ref{lem:Mhat-Mhat}, \ref{lem:Mo}, 
\begin{align*}
    \E [\calE^\top G \calE] & = \trace(\MM (\MM^\top \MM)^{-1} \Mo^\top \Mo (\MM^\top \MM)^{-1} \MM^\top) \\
    & \leq r \norm{G}_{\op} \\
    & \leq r \norm{\MM \MM^\top}_{\op} \norm{(\MM^\top \MM)^{-1}}_{\op}^{2} \norm{\Mo^\top \Mo}_{\op} \\
    & \leq O\left( \frac{r}{\eta^2} \frac{\lambda_{\max}(\Sigma)}{\lambda_{\min}(\Sigma)} \norm{X_{0}}_{2}^{2} \norm{Y_{0}}_{2}^{2} \frac{1}{\eta^2 \norm{Y_{0}}_{2}^{2} \lambda_{\min}(\Sigma) N_{2}} \right) \\
    & = O \left( \frac{r}{\eta^4} \frac{\lambda_{\max}(\Sigma)}{\lambda_{\min}^{2}(\Sigma)} \norm{X_{0}}_{2}^{2} \frac{1}{N_{2}} \right),
\end{align*}
with probability at least $1-\delta_{1}$, when $N_{2} \geq O\left((r+\log 2/\delta_{1}) \left(\frac{1}{\eta^2}\frac{\lambda_{\max}(\Sigma)}{\lambda_{\min}(\Sigma)}\right)^{2}\right)$. In addition, $\norm{G}_{F} \sqrt{\log 2/\delta_{1}} \leq \sqrt{r \log 2/\delta_{1}} \norm{G}_{\op}$. Therefore,
\begin{align*}
    \calE^\top G \calE & \leq O \left( r \norm{G}_{\op} + (\log 2/\delta_{1}) \norm{G}_{\op} + \sqrt{r \log 2/\delta_{1}} \norm{G}_{\op} \right) \\
    & \leq O \left( (r+\log 2/\delta_{1}) \norm{G}_{\op} \right) \\
    & \leq O \left( \frac{r + \log 2/\delta_{1}}{\eta^4} \frac{\lambda_{\max}(\Sigma)}{\lambda_{\min}^{2}(\Sigma)} \norm{X_{0}}_{2}^{2} \frac{1}{N_{2}}\right).
\end{align*}
\end{proof}

\begin{lemma}\label{lem:Mhat-M-inverse}
Suppose each covariate $x_i$ is mean-zero, satisfies $\E[xx^\top] = \Sigma$ and $\Sigma$-subgaussian, and $\abs{Y_0^\top \hat{A^2}_i} \ge \eta \norm{Y_0}_2$ for all $i \in [r]$. When $N_{2}\geq O\left((r+\log 1/\delta_{2}) \left(\frac{1}{\eta^2}\frac{\lambda_{\max}(\Sigma)}{\lambda_{\min}(\Sigma)}\right)^2 \right)$, the matrix $\MM^\top \MM$ is invertible and 
\begin{equation*}
    \norm{(\MM^\top \MM)^{-1}}_{\op} \leq O\left( \frac{1}{\eta^2 \norm{Y_{0}}_{2}^{2} \lambda_{\min}(\Sigma)} \frac{1}{N_{2}}\right),
\end{equation*}
with probability at least $1-\delta_{1}$.
\end{lemma}

\begin{proof}
Note that $\MM = (Y_{0} \otimes \calX^\top)^\top (\hat{A}^{2} \odot \hat{A}^1 ) = (Y_{0}^\top \otimes \calX) (\hat{A}^2 \odot \hat{A}^1) = Y_{0}^\top \hat{A}^{2} \odot \calX \hat{A}^{1}$. Thus, by defining matrix $U\in \R^{d_1 \times r}$ with columns $U_i = (Y_{0}^\top \hat{A}_{i}^{2})\hat{A}_{i}^{1}$, it can be written that $\MM^\top \MM = U^\top \calX^\top \calX U$. Note that the columns of $U$ are orthogonal with each other. Since $\E[\frac{1}{N_2} \MM^\top \MM] = \E[U^\top (\frac{1}{N_2} \calX^\top \calX) U] = U^\top \Sigma U$, we let $\frac{1}{N_{2}}\MM^\top \MM = \calE + U^\top \Sigma U$ with matrix $\calE$ satisfying $\E[\calE] = 0$. In addition, we have 
\begin{align*}
    & \norm{U}_{\op} = \norm{\hat{A}^{1} \textrm{diag}[(Y_{0}^\top \hat{A}_{i}^{2})_{i=1}^{r}]}_{\op} \\
    & \leq \norm{\hat{A}^{1}}_{\op} \underset{i \in [r]}{\max} |Y_{0}^\top \hat{A}_{i}^{2}| \\
    & \leq \norm{\hat{A}^{1}}_{\op} \cdot \norm{Y_{0}}_{2} \underset{i \in [r]}{\max}\norm{\hat{A}_{i}^{2}}_{2} \\
    & \leq \norm{\hat{A}^{1}}_{\op} \cdot \norm{Y_{0}}_{2} \left( \underset{i \in [r]}{\sum} \norm{\hat{A}_{i}^{2}}_{2}^{2} \right)^{\frac{1}{2}} \\
    & = 1\cdot \norm{Y_{0}}_{2} \cdot 1 = \norm{Y_{0}}_{2}.
\end{align*}
Thus, applying Lemma \ref{lem:concentration-proj-covariance}, we conclude that as long as $N_{2} \geq O(r+\log (1/\delta_1 ))$ we have 
$\norm{\calE}_{\op} \leq O\left(\lambda_{\max}(\Sigma) \norm{Y_{0}}_{2}^{2} \left(\sqrt{\frac{r+\log (1/\delta_1)}{N_{2}}}\right)\right)$ with probability at least $1-\delta_{1}$. Besides,
\begin{equation*}
    \lambda_{\min}(U^\top \Sigma U) = \min_{x \in \R^{r}} \frac{x^\top U^\top \Sigma U x}{x^\top x} \geq \eta^2  \norm{Y_{0}}_{2}^{2} \min_{\omega \in \R^{d_{1}}} \frac{\omega^\top \Sigma \omega}{\omega^\top \omega} = \eta^2 \norm{Y_{0}}_{2}^{2} \lambda_{\min}(\Sigma),
\end{equation*}
where the first inequality follows from substituting $\omega = Ux$ and observing
\begin{equation*}
    \omega^\top \omega = x^\top U^\top U x \geq \min_{i} |Y_{0}^\top \hat{A}_{i}^{2}|^2 x^\top x \geq \eta^2 \norm{Y_{0}}_{2}^{2} x^\top x.
\end{equation*}
Therefore,
\begin{align*}
    & \lambda_{\min}(\frac{1}{N_{2}}\MM^\top \MM) \geq \lambda_{\min}(U^\top \Sigma U) - \lambda_{\max}(\calE) \\
    & \geq \eta^2 \norm{Y_{0}}_{2}^{2} \lambda_{\min}(\Sigma) - \norm{\calE}_{\op} \\
    & \geq O\left(\eta^2 \norm{Y_{0}}_{2}^{2} \lambda_{\min}(\Sigma) - \lambda_{\max}(\Sigma)\norm{Y_{0}}_{2}^{2}  \sqrt{\frac{r+\log 1/\delta_{1}}{N_{2}}} \right) \\
    & = O\left(\norm{Y_{0}}_{2}^{2} \left(\eta^2 \lambda_{\min}(\Sigma) - \lambda_{\max}(\Sigma) \sqrt{\frac{r+\log 1/\delta_{1}}{N_{2}}} \right)\right).
\end{align*}
Therefore, as long as $N_{2} \geq O\left( (r+\log 1/\delta_{1}) \left( \frac{\lambda_{\max}(\Sigma)}{\eta^2 \lambda_{\min}(\Sigma)} \right)^{2} \right)$, $\frac{1}{N_{2}} \MM^\top \MM$ is invertible and so is $\MM^\top \MM$. Now, $(\MM^\top\MM)^{-1} = \frac{1}{N_{2}}(\frac{1}{N_{2}}\MM^\top \MM)^{-1} = \frac{1}{N_{2}}(\calE + U^\top \Sigma U)^{-1}$. Moreover,
\begin{equation*}
    \norm{(U^\top \Sigma U)^{-1}\calE}_{\op} \leq \norm{(U^\top \Sigma U)^{-1}}_{\op}\norm{\calE}_{\op} \leq \frac{\norm{\calE}_{\op}}{\lambda_{\min}(U^\top \Sigma U)} \leq O\left(\frac{\lambda_{\max}(\Sigma)}{\lambda_{\min}(\Sigma)} \frac{1}{\eta^2} \sqrt{\frac{r+ \log 1/\delta_{1}}{N_{2}}}\right).
\end{equation*}
Therefore, as long as $N_{2}\geq O\left((r+\log 1/\delta_{2}) \left(\frac{1}{\eta^2}\frac{\lambda_{\max}(\Sigma)}{\lambda_{\min}(\Sigma)}\right)^2 \right)$, we have $\norm{(U^\top \Sigma U)^{-1} \calE}_{\op} \leq 1/4$. Finally, applying Lemma \ref{lem:inverse-formula} we have $(\frac{1}{N_{2}} \MM^\top \MM)^{-1} = (U^\top \Sigma U)^{-1} + F$, where $\norm{F}_{\op} \leq \frac{1}{3} \norm{(U^\top \Sigma U)^{-1}}_{\op}$, and
\begin{equation*}
    \norm{(\MM^\top \MM)^{-1}} \leq \frac{4}{3N_{2}} \norm{(U^\top \Sigma U)^{-1}}_{\op} \leq \frac{4}{3N_{2}} \frac{1}{\eta^2 \norm{Y_{0}}_{2}^{2} \lambda_{\min}(\Sigma)}.
\end{equation*}
\end{proof}

\begin{lemma}\label{lem:MM-M}
Suppose each covariate $x_i$ is mean-zero, satisfies $\E[xx^\top] = \Sigma$ and $\Sigma$-subgaussian, and $\max \{\sin \theta(\hat{A}^{1}, A^{1}), \sin \theta (\hat{A}^{2}, A^{2})\} \leq \delta$, then if $N_{2}\geq O(r+\log 1/\delta_{1})$, we have
\begin{equation*}
    \norm{\MM - M}_{\op} \leq O\left( \norm{Y_{0}}_{2} \delta \sqrt{N_{2} r \lambda_{max}(\Sigma)} \right),
\end{equation*}
with probability at least $1-\delta_{1}$.
\end{lemma}

\begin{proof}
Write $\hat{A}^{1} = A^{1} + E^1$, $\hat{A}^2 = A^2 + E^2$. Note that $\hat{A}^{1}_{\perp}\hat{A}^{1\top}_{\perp} + \hat{A}^1 \hat{A}^{1\top} = I_{d_{1}}$, we have 
\begin{align*}
    \norm{A^1 - \hat{A}^1}_{\op} & = \norm{\hat{A}_{\perp}^{1}\hat{A}_{\perp}^{1\top}A^{1} - \hat{A}^{1}(I_{r} - \hat{A}^{1\top}A^{1})}_{\op} \leq \norm{\hat{A}_{\perp}^{1}}_{\op} \norm{\hat{A}_{\perp}^{1\top}A^{1}}_{\op} + \norm{\hat{A}^{1}}_{\op} \norm{I_{r} - \hat{A}^{1\top}A^{1}}_{\op} \\
    & \leq 1 \cdot \sin \theta(\hat{A}^{1}, A^{1}) + 1 \cdot \sin^{2} \theta(\hat{A}^{1}, A^1 ) = O(\delta),
\end{align*}
where the last inequality is due to
\begin{align*}
    \norm{I_{r} - \hat{A}^{1\top}A^{1}}_{\op} & = \lambda_{\max}(I_{r} - \hat{A}^{1\top}A^{1}) = 1 - \lambda_{\min}(\hat{A}^{1\top}A^{1}) = 1 - \cos \theta(\hat{A}^{1}, A^{1}) \\
    & = \frac{1 - \cos^{2}\theta(\hat{A}^{1}, A^{1})}{1 + \cos \theta(\hat{A}^{1}, A^{1})} = \frac{\sin^2 \theta (\hat{A}^{1}, A^{1})}{1 + \cos \theta(\hat{A}^{1}, A^{1})} \leq \sin^{2} \theta(\hat{A}^{1}, A^1) = \delta^2,
\end{align*}
where $\theta(\hat{A}^{1}, A^{1})$ is the principal angle between column subspaces of $\hat{A}^{1}$ and $A^{1}$. Therefore, $\norm{E^{1}}_{\op} \leq O(\delta)$, in the same way, $\norm{E^{2}}_{\op} \leq O(\delta)$. Now,
\begin{align*}
    \norm{\MM - M}_{\op} & = \norm{(Y_{0}\otimes \calX^{\top})^{\top} (\hat{A}^{2} \odot \hat{A}^{1}) - (Y_{0} \otimes \calX^{\top})^{\top}(A^2 \odot A^1)}_{\op} \\
    & = \norm{(Y_{0}^{\top}\hat{A}^{2} \odot \calX \hat{A}^{1}) - (Y_{0}^\top A^{2} \odot \calX A^{1})}_{\op} \\
    & = \norm{(Y_{0}^\top (A^2 + E^2 )\odot \calX (A^1 + E^1)) - (Y_{0}^{\top} A^2 \odot \calX A^1)}_{\op} \\
    & \leq \norm{Y_{0}^\top A^2 \odot \calX E^1}_{\op} + \norm{Y_{0}^\top E^2 \odot \calX A^1}_{\op} + \norm{Y_{0}^\top E^{2} \odot \calX E^{1}}_{\op}.
\end{align*}
Consider the first term,
\begin{align*}
    & \norm{Y_{0}^\top A^2 \odot \calX E^1}_{\op}^{2} = \lambda_{\max}((Y_{0}^\top A^2 \odot \calX E^{1})^\top (Y_{0}^\top A^2 \odot \calX E^1)) \\
    & = \lambda_{\max}(U^\top \calX^\top \calX U) = N_{2} \lambda_{\max}(U^\top (\frac{1}{N_{2}} \calX^\top \calX)U),
\end{align*}
where $U\in \R^{d_{1}\times r}$ has columns $U_{i} = (Y_{0}^\top A_{i}^{2})E_{i}^{1}$ with $A_{i}^{2}$ and $E_{i}^{1}$ being columns of $A^{2}$ and $E^{1}$. Note that $\norm{U}_{\op} \leq \norm{Y_{0}}_{2} \norm{E^{1}}_{F} \leq \norm{Y_{0}}_{2} \sqrt{r} \norm{E^1}_{\op} = O(\sqrt{r} \norm{Y_{0}}_{2} \delta)$, by Lemma \ref{lem:concentration-proj-covariance}, $\lambda_{\max}(U^\top (\frac{1}{N_{2}} \calX^\top \calX)U)$ is uppper bounded by $O\left( \norm{U^\top \Sigma U}_{\op} + \lambda_{\max}(\Sigma) r \norm{Y_{0}}_{2}^{2} \delta^2 \sqrt{\frac{r+\log 1/\delta_{1}}{N_{2}}} \right)$.  Moreover, $\norm{U^\top \Sigma U}_{\op} \leq \lambda_{\max}(\Sigma) \norm{U}_{\op}^{2} \leq \lambda(\Sigma)r\norm{Y_{0}}_{2}^{2}\delta^2$. Thus, when $N_{2} \geq O(r+\log 1/2\delta_{1})$, $\lambda_{max}(U^\top (\frac{1}{N_{2}}\calX^\top \calX)U) \leq O(\lambda_{\max}(\Sigma) r \norm{Y_{0}}_{2}^{2} \delta^2)$. Therefore, $\norm{Y_{0}^\top A^2 \odot \calX E^{1}}_{\op} \leq O\left( \delta \norm{Y_{0}}_{2} \sqrt{N_{2} r \lambda_{\max}(\Sigma)} \right)$.

The second term $\norm{Y_{0}^\top E^2 \odot \calX A^1}_{\op}$ and the third term $\norm{Y_{0}^\top E^{2} \odot \calX E^{1}}_{\op}$ can be shown in the same way having an upper bound of the same magnitude.
\end{proof}

\begin{lemma}\label{lem:Mhat-Mhat}
Suppose each covariate $x_i$ is mean-zero, satisfies $\E[xx^\top] = \Sigma$ and $\Sigma$-subgaussian, and $N_{2} \geq O(r + \log 1/\delta_{1})$. Then
\begin{equation*}
    \norm{\MM\MM^\top}_{\op} \leq O\left( N_{2} \lambda_{\max}(\Sigma) \norm{Y_{0}}_{2}^{2} \right),
\end{equation*}
with probability at least $1-\delta_{1}$.
\end{lemma}
\begin{proof}
Write 
\begin{align*}
    \norm{\MM\MM^\top}_{\op} & = \lambda_{\max}(\MM^\top \MM) = \lambda_{\max}((Y_{0}^\top\hat{A}^{2} \odot \calX \hat{A}^{1})^\top (Y_{0}^\top \hat{A}^{2} \odot \calX \hat{A}^{1})) \\
    & = \lambda_{\max}(U^\top (\calX^\top \calX) U) = N_{2} \lambda_{\max}(U^\top (\frac{1}{N_{2}}\calX^\top \calX)U),
\end{align*}
where $U$ has orthogonal columns $U_{i} = (Y_{0}^\top \hat{A}_{i}^{2})\hat{A}_{i}^{1}$. Since $\norm{U}_{\op} \leq \norm{Y_{0}}_{2}\underset{i \in [r]}{\max} \norm{\hat{A}_{i}^{2}}_{2} \leq \norm{Y_{0}}_{2}$. By Lemma \ref{lem:concentration-proj-covariance}, when $N_{2} \geq O(r+\log 1/\delta_{1})$,
\begin{align*}
    \lambda_{\max}(U^\top (\frac{1}{N_{2}}\calX^\top \calX)U) \leq O \left( \norm{U^\top \Sigma U}_{\op} + \lambda_{\max}(\Sigma) \norm{Y_{0}}_{2}^{2}\sqrt{\frac{r+\log 1/\delta_{1}}{N_{2}}} \right) \leq O(\lambda_{\max}(\Sigma)\norm{Y_{0}}_{2}^{2}).
\end{align*}
Therefore, $\norm{\MM\MM^\top}_{\op} \leq O(N_{2}\lambda_{\max}(\Sigma)\norm{Y_{0}}_{2}^{2})$ with probability at least $1-\delta_{1}$.
\end{proof}

\begin{lemma}\label{lem:Mo-M0}
If $\max \{\sin \theta(\hat{A}^{1}, A^{1}), \sin \theta (\hat{A}^{2}, A^{2})\} \leq \delta$, then
\begin{equation*}
    \norm{\Mo - M_{0}}_{\op} \leq O\left(\sqrt{r} \delta \norm{X_{0}}_{2} \norm{Y_{0}}_{2}\right).
\end{equation*}
\end{lemma}
\begin{proof}
Let $\hat{A}^{2} = A^{2} + E^{2}$, $\hat{A}^{1} = A^{1} + E^{1}$, then \begin{equation*}
    \norm{\Mo - M_{0}}_{\op} \leq \norm{Y_{0}^\top A^{2} \odot X_{0}^\top E^{1}}_{\op} + \norm{Y_{0}^\top E^{2} \odot X_{0}^\top A^{1}}_{\op} + \norm{Y_{0}^\top E^{2} \odot X_{0}^\top E^{1}}_{\op}.
\end{equation*}
The first term $\norm{Y_{0}^\top A^{2} \odot X_{0}^\top E^{1}}_{\op}^{2} = \lambda_{\max}(U^\top X_{0} X_{0}^\top U)$, where $U$ has columns $U_{i} = (Y_{0}^\top A_{i}^{2})E_{i}^{1}$. Using the upper bound of $\norm{E^{1}}_{\op}$ in Lemma \ref{lem:MM-M} we have $\norm{U}_{\op} \leq \norm{Y_{0}}_{2} \norm{E^{1}}_{F} \leq O(\sqrt{r} \norm{Y_{0}}_{2} \delta)$. Therefore,
\begin{equation*}
    \norm{Y_{0}^\top A^{2} \odot X_{0}^\top E^{1}}_{2} \leq \sqrt{\norm{U}_{\op}^{2} \norm{X_{0}}_{2}^{2}} \leq O\left( \sqrt{r} \norm{X_{0}}_{2} \norm{Y_{0}}_{2} \delta \right).
\end{equation*}
Thus, $\norm{\Mo - M_{0}}_{\op} \leq O(\sqrt{r} \delta \norm{X_{0}}_{2} \norm{Y_{0}}_{2})$. 
\end{proof}

\begin{lemma}\label{lem:Mo}
\begin{equation*}
    \norm{\Mo}_{\op} \leq \norm{X_{0}}_{2} \norm{Y_{0}}_{2}.
\end{equation*}
\end{lemma}
\begin{proof}
\begin{equation*}
    \norm{\Mo}_{\op} = \norm{Y_{0}^\top \hat{A}^{2} \odot X_{0}^\top \hat{A}^{1}}_{\op} = \norm{X_{0}^\top U}_{\op} \leq \norm{X_{0}}_{2} \norm{U}_{\op} \leq \norm{X_{0}}_{2} \norm{Y_{0}}_{2},
\end{equation*}
where $U$ has columns $U_{i} = (Y_{0}^\top \hat{A}_{i}^{2})\hat{A}_{i}^{1}$.
\end{proof}

\end{document}